\documentclass[12pt]{article}
\usepackage[tone, extra, safe]{tipa}
\usepackage{qtree}
\usepackage{xcolor}
\usepackage{pgf}
\usepackage{tikz}
\usetikzlibrary{arrows,positioning,fit,shapes.multipart,matrix,chains}
\usepackage[utf8x]{inputenc}
\usepackage{pifont}
\usepackage{fixltx2e}
\usepackage{colortbl}
\usepackage{arydshln}
\usepackage{rotating}
\usepackage{soul}
\usepackage[colorlinks=true,citecolor=red]{hyperref}
\usepackage{tree-dvips}
\usepackage{cgloss4e}
\usepackage{linguex}
\usepackage{dsfont,amsmath,amsfonts,amsthm,amssymb}

\newcommand{\tab}{\begin{tabular}[c]{@{}l@{}}}
\newcommand{\etab}{\end{tabular}}
\newcommand{\N}{\mathbb{N}}
\newcommand{\p}{\mathbb{P}}
\newcommand{\seg}[1]{[\![#1]\!]}

\begin{document}

\newtheorem{theo}{Theorem}[section]
\newtheorem{lemme}{Lemma}[section]
\newtheorem{prop}{Proposition}[section]
\newtheorem{cor}{Corollary}[section]
\newtheorem{defi}{Definition}[section]

\newsavebox{\movingonce}
\newsavebox{\moving}
\newsavebox{\head}
\newsavebox{\movingoncem}
\newsavebox{\movingm}
\newsavebox{\headm}
\setbox\movingonce=\hbox{\qroof{$\lambda_2:\gamma_2\cdot-\text{f}$}.{$</>$} }
\setbox\moving=\hbox{\qroof{$\lambda_2:\gamma_2\cdot-\text{f}\ \delta_2$}.{$</>$} }
\setbox\head=\hbox{\qroof{$\lambda_1:\gamma_1\cdot+\text{f}\ \delta_1$}.{$</>$} }
\setbox\movingoncem=\hbox{\qroof{$\lambda_2:\gamma_2-\text{f}\cdot$}.{$</>$} }
\setbox\movingm=\hbox{\qroof{$\lambda_2:\gamma_2-\text{f}\cdot\delta_2$}.{$</>$} }
\setbox\headm=\hbox{\qroof{$\lambda_1:\gamma_1+\text{f}\cdot\delta_1$}.{$</>$} }

\title{\sc A probabilistic top-down parser\\
for minimalist grammars}
\author{T. Mainguy\\\\
this paper was written during an internship\\supervised by E. Stabler (UCLA) and O. Catoni (ENS)}
\date{\today}

\maketitle

\abstract{This paper describes a probabilistic top-down parser for minimalist grammars. Top-down parsers have the great advantage of having a certain predictive power during the parsing, which takes place in a left-to-right reading of the sentence. Such parsers have already been well-implemented and studied in the case of Context-Free Grammars (see for example \cite{Roa01}), which are already top-down, but these are difficult to adapt to Minimalist Grammars, which generate sentences bottom-up. I propose here a way of rewriting Minimalist Grammars as Linear Context-Free Rewriting Systems, allowing us to easily create a top-down parser. This rewriting allows also to put a probabilistic field on these grammars, which can be used to accelerate the parser. I propose also a method of refining the probabilistic field by using algorithms used in data compression.}

\newpage

\tableofcontents

\newpage

Throughout this paper, I will refer as a \emph{subtree} of a tree $\mathcal{T}$, the set of nodes in $\mathcal{T}$ dominated by a particular node, which will be the root of the subtree. On the other hand, a cut is the set of the leaves of a finite prefix tree of $\mathcal{T}$.

\section{Introduction}

The idea of this parser is to see a minimalist grammar (MG) as a linear context-free rewriting system (LCFRS) on its derivation trees. This transformation allows us to work on a grammar without movement, generating sentences from top to bottom (on contrary of MG, which generates sentences bottom-up), and to put a probabilistic field on it.

\subsection{Minimalist grammars}
\label{MG}

Minimalist grammars are designed to generate (subparts of) human natural languages. They are framed in Chomsky's minimalist program \cite{Cho95}, and were first described by E. Stabler in \cite{Sta97}. For the sake of clarity, I will in this paper use slightly different convention to represent the trees generated by a minimalist grammar.

Minimalist grammars distinguishe themselves from more classical context-free grammars by the fact that they allow \emph{movement}, commonly required by syntacticians to generate such sentences as (for example) `Which mouse did the cat eat', where `which mouse' is base-generated at the end of the sentence (in the object position), and moves at the front. The tree corresponding to this sentence, as generated by the toy Minimalist Grammar we will consider here as example, is the following:

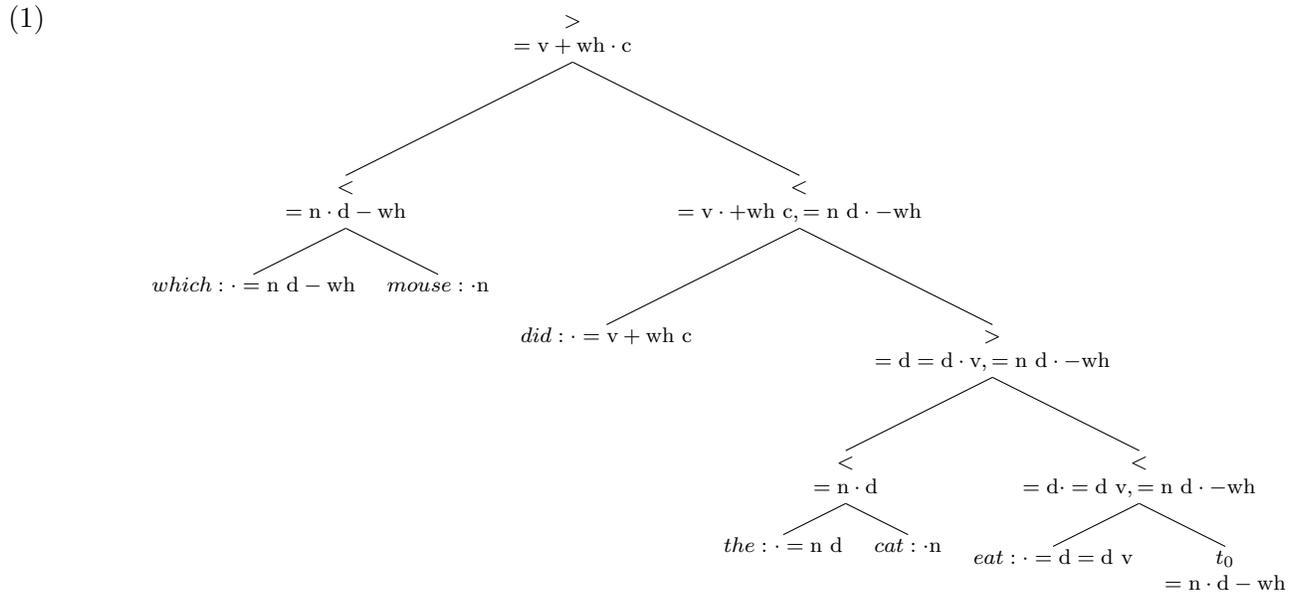
\ex.
\label{which mouse}
{\scriptsize
\Tree
  [.{$>$\\$=\text{v}+\text{wh}\cdot\text{c}$}
    [.{$<$\\$=\text{n}\cdot\text{d}-\text{wh}$}
      [.{$which : \cdot=\text{n}\ \text{d}-\text{wh}$}
      ]
      [.{$mouse : \cdot\text{n}$}
      ]
    ]
    [.{$<$\\$=\text{v}\cdot+\text{wh}\ \text{c},=\text{n}\ \text{d}\cdot-\text{wh}$}
      [.{$did : \cdot=\text{v}+\text{wh}\ \text{c}$}
      ]
      [.{$>$\\$=\text{d}=\text{d}\cdot\text{v},=\text{n}\ \text{d}\cdot-\text{wh}$}
	[.{$<$\\$=\text{n}\cdot\text{d}$}
	  [.{$the : \cdot=\text{n}\ \text{d}$}
	  ]
	  [.{$cat : \cdot\text{n}$}
	  ]
	]
	[.{$<$\\$=\text{d}\cdot=\text{d}\ \text{v},=\text{n}\ \text{d}\cdot-\text{wh}$}
	  [.{$eat : \cdot=\text{d}=\text{d}\ \text{v}$}
	  ]
	  [.{$t_0$\\$=\text{n}\cdot\text{d}-\text{wh}$}
	  ]
	]
      ]
    ]   
  ]} 
where $t_0$ denotes the \emph{trace} of the subtree `which mouse', which has moved in front of the sentence. A trace is kept for psychological reasons, as these traces can be shown to be still present for the computation of the meaning of sentences. They also allows to keep what is called the \emph{deep structure}, corresponding to the tree where no movement happened, and all constituents are in their base position, where lexical selection takes place.

A minimalist grammar takes several \emph{lexical elements}, and builds a tree with them. The toy grammar we consider will have the following lexical items:

\ex.
\label{catsandmouses}
\begin{itemize}
 \item $mouse :: \text{n}$
 \item $cat :: \text{n}$
 \item $the :: =\text{n}\ \text{d}$
 \item $which :: =\text{n}\ \text{d}-\text{wh}$
 \item $ate :: =\text{d}=\text{d}\ \text{c}$
 \item $eat :: =\text{d}=\text{d}\ \text{v}$
 \item $did :: =\text{v}+\text{wh}\ \text{c}$
 \item $did :: =\text{v}\ \text{c}$
\end{itemize}

This grammar generates (roughly) all affirmative/interrogative past sentences about a cat and a mouse eating each other.

As can be seen, many symbols are used next to the actual \emph{phonetic contents} of the words (the, eat, cat, etc...). These are \emph{syntactic features}, and the sequence of these in a lexical item represents its \emph{syntactic category}, and is all that is needed to compute the tree. Two lexical items with the same lexical category can be freely interchanged without losing grammaticality.

The \emph{syntactic features} may be of four types:
\begin{itemize}
 \item[-] \emph{categories}, represented by a string of letters, among which is the \emph{distinguished feature} c, used to recognise the grammatical outputs. For example, n. The set of categories will be noted Cat.
 \item[-] \emph{selectors}, represented by the string of letters of a category, preceded by a =. For example, =n. The set of selectors will be noted Sel.
 \item[-] \emph{licensees}, represented by a string of letters preceded by a -. For example, -wh. The set of licensees will be noted Licensee.
 \item[-] \emph{licensors}, represented by a string of letters corresponding to a licencee, preceded by a +. For example, +wh. The set of licensors will be noted Licensor.
\end{itemize}

Syntactic features must follow a certain order, to ensure good formation of trees : $\text{Syn}=(\text{Select}(\text{Select} \cup\text{Licensor})^*) \text{Cat}\text{Licensee}^*$

The trees are computed by using two functions on the lexical items, to form \emph{constituents} (i.e., trees):
\begin{itemize}
 \item[-] \emph{merge}, when a selector selects a corresponding category,
 \item[-] \emph{move}, when a licensee moves to a corresponding licensor.
\end{itemize}

Let's see how this works on our little tree:

\begin{itemize}
 \item Take $which : \cdot=\text{n}\ \text{d}-\text{wh}$ and $mouse : \cdot\text{n}$. We added a $\cdot$ in front of the syntactic categories to keep track of the derivation. Here, the two features just right of the dot (the \emph{current features} are =n and n. It's a selector and its corresponding category, so we can \emph{merge} them to a bigger constituent:

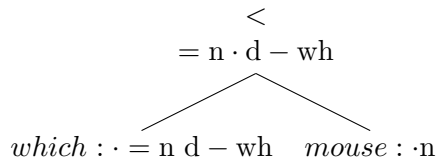
\Tree
  [.{$<$\\$=\text{n}\cdot\text{d}-\text{wh}$}
    [.{$which : \cdot=\text{n}\ \text{d}-\text{wh}$}
    ]
    [.{$mouse : \cdot\text{n}$}
    ]
  ]

Both of the syntactic categories are copied to the root of the new constituent, with their dots moved one step right, since the current features were used. The category with the selector always comes first. If, as it is the case for the syntactic category of mouse, the dot ends up at the far right, the category may be left out, since it won't have any role in the further derivations. The $<$ indicates the \emph{head} of the constituent, i.e. the constituent where the selector came from.

  \item Then merging the new constituent, whose syntactic category is $=\text{n}\cdot\text{d}-\text{wh}$, with $eat : \cdot=\text{d}=\text{d}\text{v}$ (note the selector =d, corresponding to the current d) gives:

\Tree
  [.{$<$\\$=\text{d}\cdot=\text{d}\ \text{v}, =\text{n}\ \text{d}\cdot-\text{wh}$}
    [.{$eat : \cdot=\text{d}=\text{d}\ \text{v}$}
    ]
    [.{$<$\\$=\text{n}\cdot\text{d}-\text{wh}$}
      [.{$which : \cdot=\text{n}\ \text{d}-\text{wh}$}
      ]
      [.{$mouse : \cdot\text{n}$}
      ]
    ]
  ]
  
Note also that here, the second syntactic category still has something right of the dot, so stays.

  \item Together with \begin{tabular}{r}
                        \Tree [.{$<$\\$=\text{n}\cdot\text{d}$} [.{$the : =\text{n}\cdot\text{d}$} ] [.{$cat : \text{n}\cdot$} ] ]
                      \end{tabular}
, it merges into:

\Tree
  [.{$>$\\$=\text{d}=\text{d}\cdot\text{v}, =\text{n}\ \text{d}\cdot-\text{wh}$}
    [.{$<$\\$=\text{n}\cdot\text{d}$}
      [.{$the : \cdot=\text{n}\ \text{d}$}
      ]
      [.{$cat : \cdot\text{n}$}
      ]
    ]
    [.{$<$\\$=\text{d}\cdot=\text{d}\ \text{v}, =\text{n}\ \text{d}\cdot-\text{wh}$}
      [.{$eat : \cdot=\text{d}=\text{d}\ \text{v}$}
      ]
      [.{$<$\\$=\text{n}\cdot\text{d}-\text{wh}$}
	[.{$which : =\text{n}\ \text{d}\cdot-\text{wh}$}
	]
	[.{$mouse : \text{n}\cdot$}
	]
      ]
    ]
  ],

Note that for merging, only the first category in the list of syntactic categories is considered. Here, in $=\text{d}\cdot=\text{d}\ \text{v}, =\text{n}\ \text{d}\cdot-\text{wh}$, only $=\text{d}\cdot=\text{d}\ \text{v}$ is considered (in fact, only $\cdot=\text{d}$).

Note also that if the constituent with the selector is \emph{complex} (i.e. is not formed of a single lexical item), as it is here the case, the merging happens in the other way: head right and selected constituent left. This has to do with the fact that english is a SVO language.

\newpage

  \item It can then merge with $did : \cdot=\text{v}+\text{wh}\ \text{c}$, giving:
  
{\footnotesize
\Tree
  [.{$<$\\$=\text{v}\cdot+\text{wh}\ \text{c}, =\text{n}\ \text{d}\cdot-\text{wh}$}
    [.{$did :\cdot =\text{v}+\text{wh}\ \text{c}$}
    ]
    [.{$>$\\$=\text{d}=\text{d}\cdot\text{v}, =\text{n}\ \text{d}\cdot-\text{wh}$}
      [.{$<$\\$=\text{n}\cdot\text{d}$}
	[.{$the : \cdot=\text{n}\ \text{d}$}
	]
	[.{$cat : \cdot\text{n}$}
	]
      ]
      [.{$<$\\$=\text{d}\cdot=\text{d}\ \text{v}, =\text{n}\ \text{d}\cdot-\text{wh}$}
	[.{$eat : \cdot=\text{d}=\text{d}\ \text{v}$}
	]
	[.{$<$\\$=\text{n}\cdot\text{d}-\text{wh}$}
	  [.{$which :\cdot =\text{n}\ \text{d}-\text{wh}$}
	  ]
	  [.{$mouse : \cdot\text{n}$}
	  ]
	]
      ]
    ]
  ],}

  \item Finally, we can apply \emph{move}. this function applies to a \emph{single constituent}, whose first syntactic category has a \emph{licensor} right of the dot. Here, +wh. It will then scan the other categories to find a corresponding licensee right of a dot (here, $\cdot$-wh), and move the corresponding constituent to the top of the tree, giving the final sentence:

{\scriptsize
\Tree
  [.{$>$\\$=\text{v}+\text{wh}\cdot\text{c}$}
    [.{$<$\\$=\text{n}\cdot\text{d}-\text{wh}$}
      [.{$which : \cdot=\text{n}\ \text{d}-\text{wh}$}
      ]
      [.{$mouse : \cdot\text{n}$}
      ]
    ]
    [.{$<$\\$=\text{v}\cdot+\text{wh}\ \text{c},=\text{n}\ \text{d}\cdot-\text{wh}$}
      [.{$did : \cdot=\text{v}+\text{wh}\ \text{c}$}
      ]
      [.{$>$\\$=\text{d}=\text{d}\cdot\text{v},=\text{n}\ \text{d}\cdot-\text{wh}$}
	[.{$<$\\$=\text{n}\cdot\text{d}$}
	  [.{$the : \cdot=\text{n}\ \text{d}$}
	  ]
	  [.{$cat : \cdot\text{n}$}
	  ]
	]
	[.{$<$\\$=\text{d}\cdot=\text{d}\ \text{v},=\text{n}\ \text{d}\cdot-\text{wh}$}
	  [.{$eat : \cdot=\text{d}=\text{d}\ \text{v}$}
	  ]
	  [.{$t_0$\\$=\text{n}\cdot\text{d}-\text{wh}$}
	  ]
	]
      ]
    ]   
  ]}

The dots are moved as usual, the $>$ indicates the constituent where the licensor was, and in this case, since the only feature right of a dot is the distinguished feature c, we know that the derivation yielded a grammatical output.

\newpage
The constituent moved corresponds to the biggest constituent whose head is the lexical element containing the considered lincensee (this corresponds to the syntactic notion of maximal projection).
\end{itemize}

The phonetical content of this sentence is the concatenation of the phonetic contents of its leaves, in left-to-right reading: `which mouse did the cat eat'.

The notion of \emph{head} is an important notion in linguistics, since, by the principle of \emph{locality of selection}, we want to restrict the amount of information that an item has access to. As such, the only information about a constituent accessible from outside (for a merge operation, for example), is the right-of-the-dot features of its head, that is, the features of the first syntactic category (minus the left-of-the-dot ones, kept only for the sake of historic bookkeeping).

\subsection{Derivation trees}

Another way of representing the constituents generated by a grammar is by using its \emph{derivation tree}:

\begin{defi}
 The \emph{derivation tree} of a constituent is a binary tree showing the history of its building by the functions $merge$ and $move$. Its leaves are lexical items, and its nodes are labelled by either $\bullet$ (merge, it's then a binary node) or $\circ$ (move, it's then a unary node).
\end{defi}

For example, the derivation tree of the previous example \ref{which mouse}, is:

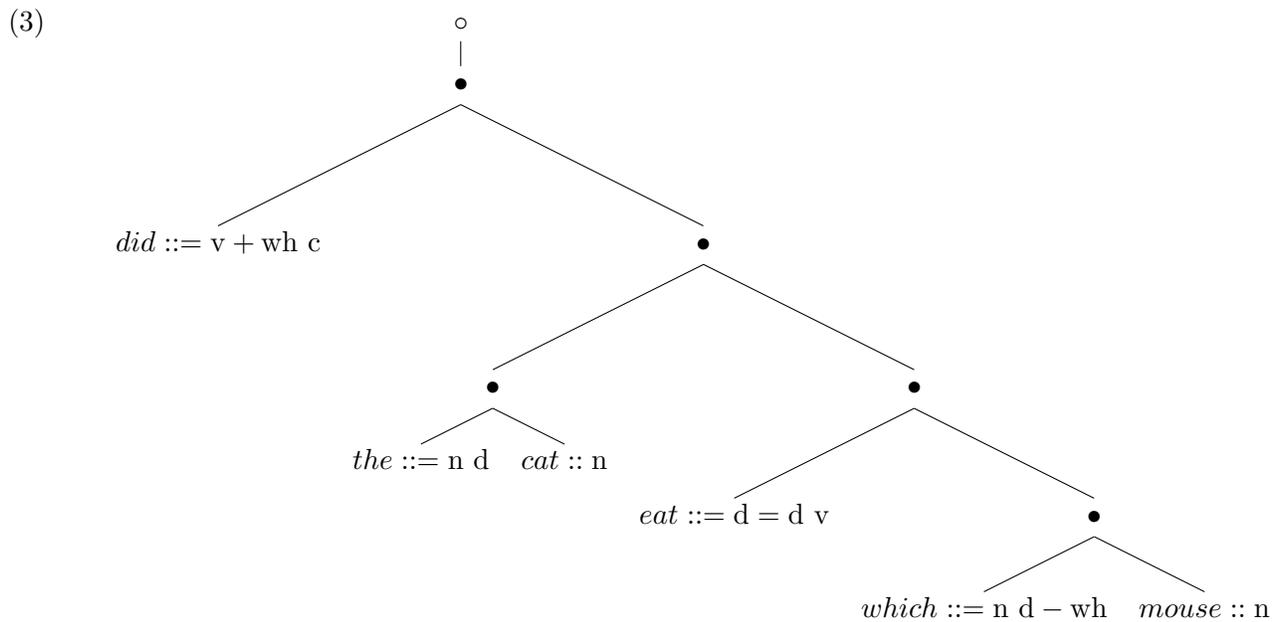
\ex.
{\small
\label{dermouse}
\Tree
  [.{$\circ$}
    [.{$\bullet$}
      [.{$did :: =\text{v}+\text{wh}\ \text{c}$}
      ]
      [.{$\bullet$}
	[.{$\bullet$}
	  [.{$the :: =\text{n}\ \text{d}$}
	  ]
	  [.{$cat :: \text{n}$}
	  ]
	]
	[.{$\bullet$}
	  [.{$eat :: =\text{d}=\text{d}\ \text{v}$}
	  ]
	  [.{$\bullet$}
	    [.{$which :: =\text{n}\ \text{d}-\text{wh}$}
	    ]
	    [.{$mouse :: \text{n}$}
	    ]
	  ]
	]
      ]   
    ]
  ]
  }

Let's note that to each subtree of a derivation tree corresponds a unique constituent, appearing in the construction of the main one. We can then label each node of a derivation tree by the syntactic category of the corresponding constituent.

\section{Probabilistic minimalist grammars}
\label{PMG}

\subsection{Derivation trees of MG as LCFRS-derived trees}
\label{MGCFG}

The basis of this method is to see minimalist derivation trees as trees generated by linear context-free rewriting systems (LCFRS). Putting a probability field on these is indeed very easy. A similar approach was also used in the minimalist parser of H. Harkema in his thesis \cite{Har01}.

We thus take a general \emph{minimalist grammar} $\mathcal{G}=(\sigma,\text{Feat},\text{Lex},\mathcal{F})$.

The closure of Lex under $\mathcal{F}$ gives the outputs of the grammars. Here, we will not consider these outputs, but the derivation trees describing the process giving these outputs.

One important difference between context-free grammars, for which probabilistic versions are well-studied, and minimalist grammars is that CFG generate trees from top to bottom, by means of rules rewriting each non-terminal node by a number of other nodes, while a MG generates trees from bottom to the top, by merging and moving elements. While in the CFG case, we begin with a single symbol and then choose rules to rewrite it (thus enabling us to assign probabilities to the process by assigning probabilities to the rewriting rules), in MG, we begin with a bunch of lexical items, not necessarily compatible with each other, and merge them together (and occasionally moving them too). Here we will present a way of seeing the generating process of MG as a LCFRS, which, as CFG, generates from top to bottom with a set of rules. The differents non-terminal symbols will be defined by closure of a certain set of axioms (starting symbols) under a set of inference rules, giving this way a top-down way of generating derivation trees of MG.

\subsubsection{Categories and partial outputs}

In order to do this, we will first define a particular type of objects, called categories, which will be the non-terminal symbols of our LCFRS:
\begin{defi}\label{cat}
	A category is either a lexical item, or a sequence of the form $[\gamma_0\cdot\delta_0, \ldots, \gamma_k\cdot\delta_k]$, where $\gamma_0, \ldots, \gamma_k, \delta_0, \ldots, \delta_k$ are elements of $\text{Syn}$, or a special symbol $start$.
	A \emph{simple} category is a category with its first dot at the leftmost place (and $k=0$). Otherwise, it is a \emph{complex} category. $start$ is neither simple or complex.
\end{defi}

Categories corresponds exactly to the list of syntactic categories defined in \ref{MG}, although our definition allows here categories which cannot be generated by a Minimalist Grammar. We will of course only be interested in those who are.

We then define a \emph{partial output} as a string $\Delta_1\ldots\Delta_n$ of categories. These represent a particular stage in the construction of a minimalist derivation tree by the corresponding LCFRS, the different categories being the categories of the partial derivation tree which is build.

\subsubsection{Axiom}

There is a single axiom, the category $start$.

\subsubsection{Inference rules}

These rules correspond to the rewriting rules of the Linear Context-Free Rewriting System $\mathcal{S}$ corresponding to our Minimalist Grammar $\mathcal{G}$. For each possible application of one of the functions $merge$ or $move$ of grammar $\mathcal{G}$ giving a particular category $\Delta$, there is a corresponding inference rule (which gives quite a lot of rules...). Then, given a particular category $\Delta$, the rules will tell how this particular type of tree (remember that categories describe a particular type of trees generated by the grammar) can be un-merged or un-moved into one (in case of un-move) or two (in case of un-merge) different types of trees. To this must be added the rules expanding the $start$ category, and the lexicalisation rules. The first allows us to begin with a unique symbol, instead of all categories ending with the distinguished feature. The second ones allow us to leave the lexical part of the parsing up to the last moment. Thus here we are:

\begin{enumerate}
	\item Start rules: for every lexical item $\gamma::\delta\ \text{c}$,\\
		\textbf{Start:}\\
		$\overline{start\longrightarrow[\delta\cdot\text{c}]}$
	\item Re-writing rules for complex categories:
		\begin{enumerate}
			\item \label{Unmerge} Un-merge rules: the left-hand category is of the form $[\delta =\text{x} \cdot \beta, S]$
				\begin{enumerate}
					\item	\label{simple} Cases where the selector was a simple tree $(\delta=\epsilon)$:
						\begin{enumerate}
							\item \label{Unmerge-1} For any lexical item of feature string $\gamma\ \text{x}$,\\
								\textbf{Unmerge-1:}\\
								$\overline{[=\text{x} \cdot \beta, S]\longrightarrow [\cdot =\text{x}\ \beta] [\gamma \cdot \text{x}, S]}$
							\item \label{Unmerge-3, simple} For any element $(\gamma\ \text{x}\cdot \varphi) \in S$, with $S'=S-(\gamma\ \text{x}\cdot \varphi)$,\\
								\textbf{Unmerge-3, simple:}\\
								$\overline{[=\text{x} \cdot \beta,\gamma\ \text{x}\cdot \varphi, S']\longrightarrow [\cdot =\text{x}\ \beta] [\gamma \cdot \text{x}\ \varphi, S']}$\\
								It should be noted that necessarily, $\varphi\neq\emptyset$.
						\end{enumerate}
					\item \label{complex} Cases where the selector was a complex tree:
						\begin{enumerate}
							\item \label{Unmerge-2} For any decomposition $S=U\sqcup V$, and any lexical item of feature string $\gamma\ \text{x}$,\\
								\textbf{Unmerge-2:}\\
								$\overline{[\delta=\text{x} \cdot \beta, S]\longrightarrow [\delta\cdot =\text{x}\ \beta, U] [\gamma \cdot \text{x}, V]}$
							\item \label{Unmerge-3, complex} For any element $(\gamma\ \text{x}\cdot \varphi) \in S$, and any decomposition $S=U\sqcup V\sqcup (\gamma\ \text{x}\cdot \varphi)$,\\
								\textbf{Unmerge-3, complex:}\\
								$\overline{[\delta=\text{x} \cdot \beta,\gamma\ \text{x}\cdot \varphi, S']\longrightarrow [\delta\cdot =\text{x}\ \beta, U] [\gamma \cdot \text{x}\ \varphi, V]}$\\
								As in \ref{Unmerge-3, simple}, $\varphi$ has to be non empty.
						\end{enumerate}
				\end{enumerate}
			\item \label{move} Un-move rules: the left-hand category is of the form $[\delta +\text{f} \cdot \beta, S]$
				\begin{enumerate}
					\item \label{Unmove-2} For any $(\gamma-\text{f}\cdot\varphi)\in S$ (necessarily unique by the Shortest Movement Constraint), with $S'=S-(\gamma -\text{f}\cdot \varphi)$,\\
						\textbf{Unmove-2:}\\
						$\overline{[\delta'+\text{f} \cdot \beta,\gamma-\text{f}\cdot\varphi, S']\longrightarrow [\delta'\cdot +\text{f}\ \beta, \gamma\cdot-\text{f}\ \varphi, S']}$
					\item \label{Unmove-1} If there is no $(\gamma-\text{f}\cdot\varphi)\in S$, then for any lexical item of feature string $\gamma -\text{f}$,\\
						\textbf{Unmove-1:}\\
						$\overline{[\delta'+\text{f} \cdot \beta, S]\longrightarrow [\delta'\cdot +\text{f}\ \beta, \gamma\cdot-\text{f}, S]}$
				\end{enumerate}
		\end{enumerate}
	\item \label{lexicalize} Re-writing rules for simple categories: for any lexical item $\lambda :: \beta$,\\
		\textbf{Lexicalize:}\\
		$\overline{[\cdot \beta]\longrightarrow \lambda :: \beta}$.
\end{enumerate}

\newpage

The set of \emph{relevant partial outputs} can thus be defined as the closure of the axiom $start$ under the inference rules. This set describes exactly all possible partial outputs given by the LCFRS $\mathcal{S}$, i.e. all possible strings of categories obtained by a cut through a tree generated by the LCFRS $\mathcal{S}$. Such a string correspond to a selection of outputs (not necessarily complete) generated by the minimalist grammar $\mathcal{G}$, such that they can be put together by application of $merge$ and $move$, in the same order (two categories will get merged only if they are adjacent in the string) to obtain a complete output. A relevant output is a relevant partial output consisting of only lexical items. It corresponds to grammatical sentences.

The \emph{relevant categories} are exactly the categories that appear in a relevant partial output. They correspond to the possible sets of similar partial trees generated by the grammar $\mathcal{G}$. They are in finite number, since, by the Shortest Movement Constraint, no two identical licensees can appear in the feature strings of a relevant category (omitting the first string). Thus two identical feature strings (diverging only by the position of the dot) can't appear together, and therefore the total length of all the feature strings of a relevant category is bounded by the sum of the length of all the feature strings of the lexical items, which is finite.

\subsubsection{Derivation trees}

With this formalism, we have now a quite straightformard way of defining minimalist derivation trees, in a way that enables us to put very simply probabilities on them: they are just the trees obtained by maximal application of rewriting rules to the axiom $start$. The probability is simply given by a probability field on the rules.

\subsection{Probabilities on MG derivation trees}

To define a probability field on the derivation trees of a MG, we now just have to put conditional probabilities on the rules discussed before, given the initial relevant category. The probability of a given tree will then be the product of the probabilities of the rules that generate it, as for regular probabilistic linear context-free rewriting systems. There can be however quite a lot of such rules and relevant categories, even if the MG is quite simple, but they can all be computed beforehand with the only knowledge of the grammar, thanks to the definition by closure of these categories. Indeed, we will see a simple method permitting to compute both the relevant categories and the inference rules that are needed.

It should be noted that the functions (Merge-1,2,3 and Move-1,2) having potentially given birth to a given relevant category are quite few (at most two), \emph{only if we use the dot notation}, which keeps track of a minimal part of the history of the derivation. This is why the relevant categories should include all features of the lexical item potentially heading the tree (and not just the ones on the right of the dot).

To settle things a bit, we will here illustrate this method with a little example.

\newpage

\subsection{Example : $a^nb^n$}

We will here consider the MG with the following lexical items ($\epsilon$ being the empty string):

\ex.\label{anbn}
\begin{itemize}
	\item $\epsilon::\text{c}$
	\item $\epsilon::\ =\text{a}+\text{m}\ \text{c}$
	\item $a::\ =\text{b}\ \text{a}-\text{m}$
	\item $b::\text{b}$
	\item $b::\ =\text{a}+\text{m}\ \text{b}$
\end{itemize}

This grammar generates exactly the strings of the form $a^nb^n$, $n\in\mathbb{N}$. Since this is a context-free language, we wouldn't have needed to use licensors and licencees, but for the sake of getting a language simple enough with enough rules (especially movement ones), we will work on this one.

We now want to get the relevant categories of this language, and the corresponding `context-free rules'. A quite straightforward way to obtain them is to start from the axiom $start$ and follow the inference rules to close the set of relevant categories. From $start$ we apply the schemes to get all applicable rules, apply them, get some new relevant categories, apply the schemes to get new rules, apply them, etc... Since they are in finite number, this algorithm will eventually terminate, giving us all the relevant categories and needed rules (we won't get them all, since the schemes could apply to non-relevant categories, but we don't want those in any case).

So here we go:

\begin{itemize}
	\item starting rules: we search for all lexical items whose features ends with c. There are two here, giving two different relevant categories: $\epsilon::\text{c}$ and $\epsilon::\ =\text{a}+\text{m}\ \text{c}$. We have thus two rules:
	\begin{itemize}
		\item[] \textbf{Start:} $start\longrightarrow[\cdot\text{c}]$
		\item[] \textbf{Start:} $start\longrightarrow[=\text{a}+\text{m}\cdot\text{c}]$
	\end{itemize}
	We have now two new relevant categories: $[\cdot\text{c}]$ and $[=\text{a}+\text{m}\cdot\text{c}]$. We will now write the rules with these on the left side of the arrow.
	\item $[\cdot\text{c}]$ correspond to case \ref{lexicalize}. There is but one lexical item with features c, which is $\epsilon::\text{c}$, so we have a single rule:
	\begin{itemize}
		\item[] \textbf{Lexicalize:} $[\cdot\text{c}]\longrightarrow\epsilon::\text{c}$
	\end{itemize}
	No new relevant category is created, so we can move to the next one:
	\item $[=\text{a}+\text{m}\cdot\text{c}]$ corresponds to the case \ref{move}, so we can have two possibilities. Since there is no `$S$', only the case \ref{Unmove-1} can apply. We must then look for lexical items whose last feature is $-\text{m}$. There is but one (and thus only one corresponding relevant category), $a::\ =\text{b}\ \text{a}-\text{m}$. So we have one possible rule:
	\begin{itemize}
		\item[] \textbf{Unmove-1:} $[=\text{a}+\text{m}\cdot\text{c}]\longrightarrow[=\text{a}\cdot+\text{m}\ \text{c},=\text{b}\ \text{a}\cdot-\text{m}]$
	\end{itemize}
	We have now a new relevant category, $[=\text{a}\cdot+\text{m}\ \text{c},=\text{b}\ \text{a}\cdot-\text{m}]$.
	\item $[=\text{a}\cdot+\text{m}\ \text{c},=\text{b}\ \text{a}\cdot-\text{m}]$ corresponds to case \ref{simple}. For case \ref{Unmerge-1}, we have to look for a lexical item whose last feature is a. Since there is no such item, we fall back to \ref{Unmerge-3, simple}. Here we have to look in `$S$' for  feature strings of type $\gamma\ \text{a}\cdot\varphi$. There is only one, namely $=\text{b}\ \text{a}\cdot-\text{m}$, so we have one rule:
	\begin{itemize}
		\item[] \textbf{Unmerge-3, simple:} $[=\text{a}\cdot+\text{m}\ \text{c},=\text{b}\ \text{a}\cdot-\text{m}]\longrightarrow[\cdot=\text{a}+\text{m}\ \text{c}][=\text{b}\cdot\text{a}-\text{m}]$
	\end{itemize}
	We got here two more relevant categories, $[\cdot=\text{a}+\text{m}\ \text{c}]$ and $[=\text{b}\cdot\text{a}-\text{m}]$.
	\item $[\cdot=\text{a}+\text{m}\ \text{c}]$ corresponds to case \ref{lexicalize}, and there is but one lexical item with the corresponding features, so we have one additional rule:
	\begin{itemize}
		\item[] \textbf{Lexicalize:}$[\cdot=\text{a}+\text{m}\ \text{c}]\longrightarrow\epsilon::\ =\text{a}+\text{m}\ \text{c}$
	\end{itemize}
	\item $[=\text{b}\cdot\text{a}-\text{m}]$ corresponds to case \ref{simple}. We first try case \ref{Unmerge-1}. We look for lexical items with last feature $\text{b}$. There are two such items, namely $b::\text{b}$ and $b::\ =\text{a}+\text{m}\ \text{b}$. We then have two rules:
	\begin{itemize}
		\item[] \textbf{Unmerge-1:} $[=\text{b}\cdot\text{a}-\text{m}]\longrightarrow[\cdot=\text{b}\ \text{a}-\text{m}][\cdot\text{b}]$
		\item[] \textbf{Unmerge-1:} $[=\text{b}\cdot\text{a}-\text{m}]\longrightarrow[\cdot=\text{b}\ \text{a}-\text{m}][=\text{a}+\text{m}\cdot\text{b}]$
	\end{itemize}
	Since `$S$' is here empty, case \ref{Unmerge-3, simple} can't apply, and we move on to the three newly discovered relevant categories, $[\cdot=\text{b}\ \text{a}-\text{m}]$, $[\cdot\text{b}]$ and $[=\text{a}+\text{m}\cdot\text{b}]$.
	\item $[\cdot=\text{b}\ \text{a}-\text{m}]$ is ready to be lexicalized, there is still only one corresponding lexical item, so we get the rule:
	\begin{itemize}
		\item[] \textbf{Lexicalize:} $[\cdot=\text{b}\ \text{a}-\text{m}]\longrightarrow a::\ =\text{b}\ \text{a}-\text{m}$
	\end{itemize}
	\item $[\cdot\text{b}]$ is in the same case, we thus have:
	\begin{itemize}
		\item[] \textbf{Lexicalize:} $[\cdot\text{b}]\longrightarrow b::\text{b}$
	\end{itemize}
	\item $[=\text{a}+\text{m}\cdot\text{b}]$ corresponds to the case \ref{Unmove-1}, with only one corresponding lexical item, thus the rule:
	\begin{itemize}
		\item[] \textbf{Unmove-1:} $[=\text{a}+\text{m}\cdot\text{b}]\longrightarrow[=\text{a}\cdot+\text{m}\ \text{b},=\text{b}\ \text{a}\cdot-\text{m}]$
	\end{itemize}
	\item $[=\text{a}\cdot+\text{m}\ \text{b},=\text{b}\ \text{a}\cdot-\text{m}]$ corresponds to case \ref{Unmerge-3, simple}, and we have one rule:
	\begin{itemize}
		\item[] \textbf{Unmerge-3, simple:} $[=\text{a}\cdot+\text{m}\ \text{b},=\text{b}\ \text{a}\cdot-\text{m}]\longrightarrow[\cdot=\text{a}+\text{m}\ \text{b}][=\text{b}\cdot\text{a}-\text{m}]$
	\end{itemize}
	Since $[=\text{b}\cdot\text{a}-\text{m}]$ has already been treated, we can move to the last untreated relevant category:
	\item $[\cdot=\text{a}+\text{m}\ \text{b}]$ is ready to be lexicalized:
	\begin{itemize}
		\item[] \textbf{Lexicalize:} $[\cdot=\text{a}+\text{m}\ \text{b}]\longrightarrow b::\ =\text{a}+\text{m}\ \text{b}$
	\end{itemize}
\end{itemize}

We are now ready to give probabilities to these rules, conditioned by the left-hand side. The assignment here is quite easy : apart from the two cases where there are two possible rules (axiom choice and category and $[=\text{b}\cdot\text{a}-\text{m}]$), the conditioned probability will be $1$ (there is no choice). For the two other cases, we can assign any probability $\lambda$ to one rule, and give the other a probability $1-\lambda$. We can now give the following table:

{\center
\begin{tabular}{|rlcl|}
	\hline
	$start$	&$\longrightarrow[\cdot\text{c}]$	&Start	&$\mathbb{P}(.)=\lambda$	\\
	$start$	&$\longrightarrow[=\text{a}+\text{m}\cdot\text{c}]$	&Start	&$\mathbb{P}(.)=1-\lambda$	\\
	\hline
	$[\cdot\text{c}]$	&$\longrightarrow\epsilon::\text{c}$	&Lexicalize	&$\mathbb{P}(.)=1$	\\
	\hline
	$[=\text{a}+\text{m}\cdot\text{c}]$	&$\longrightarrow[=\text{a}\cdot+\text{m}\ \text{c},\text{b}\ \text{a}\cdot-\text{m}]$	&Unmove-1	&$\mathbb{P}(.)=1$	\\
	\hline
	$[=\text{a}\cdot+\text{m}\ \text{c},\text{b}\ \text{a}\cdot-\text{m}]$	&$\longrightarrow[\cdot=\text{a}+\text{m}\ \text{c}][=\text{b}\cdot\text{a}-\text{m}]$	&Unmerge-3, simple	&$\mathbb{P}(.)=1$	\\
	\hline
	$[\cdot=\text{a}+\text{m}\ \text{c}]$	&$\longrightarrow\epsilon::\ =\text{a}+\text{m}\ \text{c}$	&Lexicalize	&$\mathbb{P}(.)=1$	\\
	\hline
	$[=\text{b}\cdot\text{a}-\text{m}]$	&$\longrightarrow[\cdot=\text{b}\ \text{a}-\text{m}][\cdot\text{b}]$	&Unmerge-1, simple	&$\mathbb{P}(.)=\mu$	\\
	$[=\text{b}\cdot\text{a}-\text{m}]$	&$\longrightarrow[\cdot=\text{b}\ \text{a}-\text{m}][=\text{a}+\text{m}\cdot\text{b}]$	&Unmerge-1, simple	&$\mathbb{P}(.)=1-\mu$	\\
	\hline
	$[\cdot=\text{b}\ \text{a}-\text{m}]$	&$\longrightarrow a::\ =\text{b}\ \text{a}-\text{m}$	&Lexicalize	&$\mathbb{P}(.)=1$	\\
	\hline
	$[\cdot\text{b}]$	&$\longrightarrow b::\text{b}$	&Lexicalize	&$\mathbb{P}(.)=1$	\\
	\hline
	$[=\text{a}+\text{m}\cdot\text{b}]$	&$\longrightarrow[=\text{a}\cdot+\text{m}\ \text{b},=\text{b}\ \text{a}\cdot-\text{m}]$	&Unmove-1	&$\mathbb{P}(.)=1$	\\
	\hline
	$[=\text{a}\cdot+\text{m}\ \text{b},=\text{b}\ \text{a}\cdot-\text{m}]$	&$\longrightarrow[\cdot=\text{a}+\text{m}\ \text{b}][=\text{b}\cdot\text{a}-\text{m}]$	&Unmerge-3, simple	&$\mathbb{P}(.)=1$	\\
	\hline
	$[\cdot=\text{a}+\text{m}\ \text{b}]$	&$\longrightarrow b::\ =\text{a}+\text{m}\ \text{b}$	&Lexicalize	&$\mathbb{P}(.)=1$	\\
	\hline
\end{tabular}}

\paragraph{}
We will now end by giving the probability of a particular derivation tree:

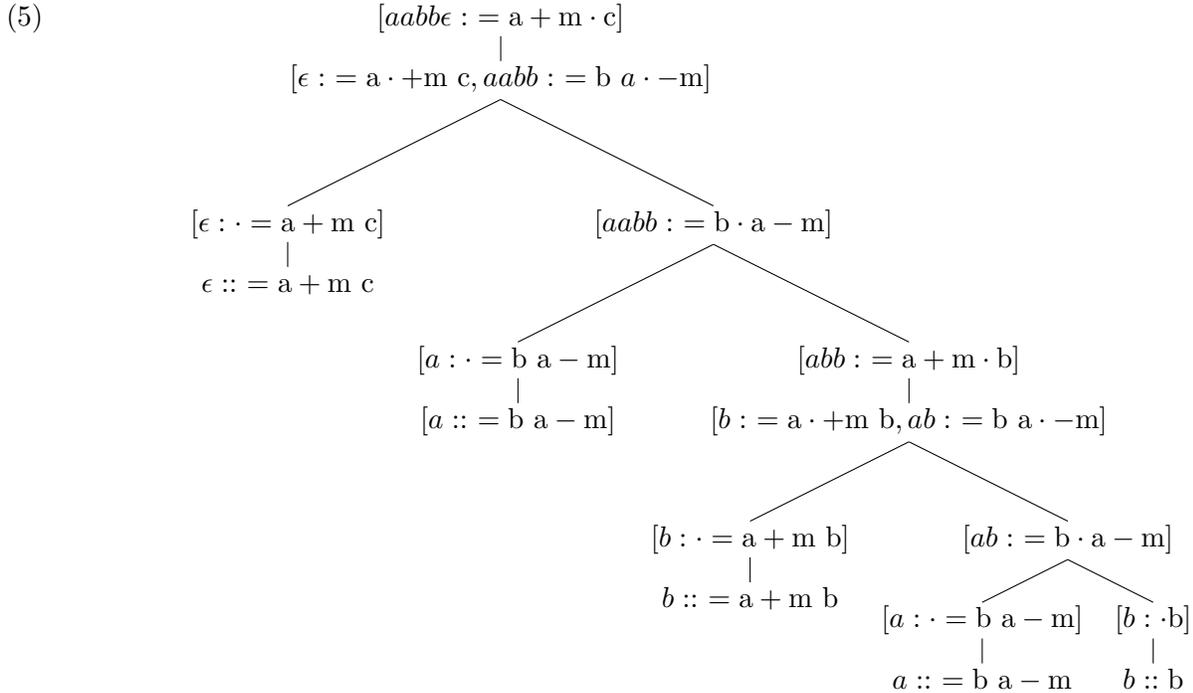
\ex.\label{anbntree}
\Tree
		[.{$[aabb\epsilon:\ =\text{a}+\text{m}\cdot \text{c}]$}
			[.{$[\epsilon:\ =\text{a}\cdot+\text{m}\ \text{c}, aabb:\ =\text{b}\ a\cdot-\text{m}]$}
				[.{$[\epsilon:\cdot=\text{a}+\text{m}\ \text{c}]$}
					[.{$\epsilon::\ =\text{a}+\text{m}\ \text{c}$}
					]
				]
				[.{$[aabb:\ =\text{b}\cdot\text{a}-\text{m}]$}
					[.{$[a:\cdot=\text{b}\ \text{a}-\text{m}]$}
						[.{$[a::\ =\text{b}\ \text{a}-\text{m}]$}
						]
					]
					[.{$[abb:\ =\text{a}+\text{m}\cdot\text{b}]$}
						[.{$[b:\ =\text{a}\cdot+\text{m}\ \text{b},ab:\ =\text{b}\ \text{a}\cdot-\text{m}]$}
							[.{$[b:\cdot=\text{a}+\text{m}\ \text{b}]$}
								[.{$b::\ =\text{a}+\text{m}\ \text{b}$}
								]
							]
							[.{$[ab:\ =\text{b}\cdot\text{a}-\text{m}]$}
								[.{$[a:\cdot=\text{b}\ \text{a}-\text{m}]$}
									[.{$a::\ =\text{b}\ \text{a}-\text{m}$}
									]
								]
								[.{$[b:\cdot\text{b}]$}
									[.{$b::\text{b}$}
									]
								]
							]
						]
					]
				]
			]
		]

All the rules here have probability $1$, except the top one, the choice of the start rule $start\longrightarrow[aabb\epsilon:\ =\text{a}+\text{m}\cdot \text{c}]$, which has probability $1-\lambda$, the one from $[aabb:\ =\text{b}\cdot\text{a}-\text{m}]$, which has probability $1-\mu$, and the one from $[ab:\ =\text{b}\cdot\text{a}-\text{m}]$, which has probability $\mu$. So the complete tree has probability $\mu(1-\mu)\lambda$, and, for example, the subtree headed by $[b:\ =\text{a}\cdot+\text{m}\ \text{b},ab:\ =\text{b}\ \text{a}\cdot-\text{m}]$ has probability $\mu$.

\newpage

\subsection{The Cats and Mouses example}

Let's now get back to our toy grammar \ref{catsandmouses} and see how it rewrites:

\ex.
\begin{itemize}
 \item $mouse :: \text{n}$
 \item $cat :: \text{n}$
 \item $the :: =\text{n}\ \text{d}$
 \item $which :: =\text{n}\ \text{d}-\text{wh}$
 \item $ate :: =\text{d}=\text{d}\ \text{c}$
 \item $eat :: =\text{d}=\text{d}\ \text{v}$
 \item $did :: =\text{v}+\text{wh}\ \text{c}$
 \item $did :: =\text{v}\ \text{c}$
\end{itemize}

The rules are the following:

{\center
\small
\begin{tabular}{|r|rlc|}
	\hline
1&	$start$	&$\longrightarrow[=\text{d}=\text{d}\cdot\text{c}]$	&Start	\\
2&	$start$	&$\longrightarrow[=\text{v}+\text{wh}\cdot\text{c}]$	&Start	\\
3&	$start$	&$\longrightarrow[=\text{v}\cdot\text{c}]$		&Start	\\
	\hline
4&	$[=\text{d}=\text{d}\cdot\text{c}]$	&$\longrightarrow[=\text{n}\cdot\text{d}][=\text{d}\cdot=\text{d}\ \text{c}]$	&Unmerge-2	\\
	\hline
5&	$[=\text{d}\cdot=\text{d}\ \text{c}]$	&$\longrightarrow[\cdot=\text{d}=\text{d}\ \text{c}][=\text{n}\cdot\text{d}]$	&Unmerge-1	\\
	\hline
6&	$[\cdot=\text{d}=\text{d}\ \text{c}]$	&$\longrightarrow ate::\ =\text{d}=\text{d}\ \text{c}$	&Lexicalize	\\
	\hline
7&	$[=\text{n}\cdot\text{d}]$	&$\longrightarrow[\cdot=\text{n}\ \text{d}][\cdot\text{n}]$	&Unmerge-1	\\
	\hline
8&	$[\cdot=\text{n}\ \text{d}]$	&$\longrightarrow the::\ =\text{n}\ \text{d}$	&Lexicalize	\\
	\hline
9&	$[\cdot\text{n}]$	&$\longrightarrow mouse::\ \text{n}$	&Lexicalize	\\
10&	$[\cdot\text{n}]$	&$\longrightarrow cat::\ \text{n}$	&Lexicalize	\\
	\hline
11&	$[=\text{v}+\text{wh}\cdot\text{c}]$	&$\longrightarrow [=\text{v}\cdot+\text{wh}\ \text{c},=\text{n}\ \text{d}\cdot-\text{wh}]$	&Unmove-1	\\
	\hline
12&	$[=\text{v}\cdot+\text{wh}\ \text{c},=\text{n}\ \text{d}\cdot-\text{wh}]$	&$\longrightarrow [\cdot=\text{v}+\text{wh}\ \text{c}][=\text{d}=\text{d}\cdot\text{v},=\text{n}\ \text{d}\cdot-\text{wh}]$	&Unmerge-1	\\
	\hline
13&	$[\cdot=\text{v}+\text{wh}\ \text{c}]$	&$\longrightarrow did::\ =\text{v}+\text{wh}\ \text{c}$	&Lexicalize	\\
	\hline
14&	$[=\text{d}=\text{d}\cdot\text{v},=\text{n}\ \text{d}\cdot-\text{wh}]$	&$\longrightarrow[=\text{n}\cdot\text{d}][=\text{d}\cdot=\text{d}\ \text{v},=\text{n}\ \text{d}\cdot-\text{wh}]$	&Unmerge-2	\\
15&	$[=\text{d}=\text{d}\cdot\text{v},=\text{n}\ \text{d}\cdot-\text{wh}]$	&$\longrightarrow[=\text{n}\cdot\text{d}-\text{wh}][=\text{d}\cdot=\text{d}\ \text{v}]$	&Unmerge-3, complex	\\
	\hline
16&	$[=\text{d}\cdot=\text{d}\ \text{v},=\text{n}\ \text{d}\cdot-\text{wh}]$	&$\longrightarrow[\cdot=\text{d}=\text{d}\ \text{v}][=\text{n}\cdot\text{d}-\text{wh}]$	&Unmerge-3, simple	\\
	\hline
17&	$[\cdot=\text{d}=\text{d}\ \text{v}]$	&$\longrightarrow eat::\ =\text{d}=\text{d}\ \text{v} $	&Lexicalize	\\
	\hline
18&	$[=\text{n}\cdot\text{d}-\text{wh}]$	&$\longrightarrow[\cdot=\text{n}\ \text{d}-\text{wh}][\cdot\text{n}]$	&Unmerge-1	\\
	\hline
19&	$[\cdot=\text{n}\ \text{d}-\text{wh}]$	&$\longrightarrow which::\ =\text{n}\ \text{d}-\text{wh}$	&Lexicalize	\\
	\hline
20&	$[=\text{d}\cdot=\text{d}\ \text{v}]$	&$\longrightarrow[\cdot=\text{d}=\text{d}\ \text{v}][=\text{n}\cdot\text{d}]$	&Unmerge-1	\\
	\hline
21&	$[=\text{v}\cdot\text{c}]$	&$\longrightarrow[\cdot=\text{v}\ \text{c}][=\text{d}=\text{d}\cdot\text{v}]$	&Unmerge-1	\\
	\hline
22&	$[\cdot=\text{v}\ \text{c}]$	&$\longrightarrow did::\ =\text{v}\ \text{c}$	&Lexicalize	\\
	\hline
23&	$[=\text{d}=\text{d}\cdot\text{v}]$	&$\longrightarrow[=\text{n}\cdot\text{d}][=\text{d}\cdot=\text{d}\ \text{v}]$	&Unmerge-2	\\
	\hline
\end{tabular}}

\newpage

\section{The probabilistic top-down parser}

The parser will work on an ordered list of hypothesis, which he will expand in turn during the parse of the sentence. Before beginning presenting the algorithm, some definitions are needed:

\subsection{Definitions}

One difficulty in working with derivation trees instead of regular derived trees, is that the order of the words cannot be easily deduced (short of redoing the actual derivation). In order to keep track of the position of a category in the \emph{derived tree} (so the parser may know in which order to expand the tree), we introduce \emph{position indices}, which denotes positions in the derived tree from its root by a chain of digits ($0$ if going down left, $1$ if going down right). From this perspective we can also define a \emph{successor} operator on them, corresponding to a left-to-right sweep of the tree.

Consider the grammar given by the following lexical items:

\ex.\label{toygrammar}
\begin{enumerate}
	\item $a::\ =\text{x}+\text{f}\ \text{c}$
	\item $b::\ =\text{y}\ \text{x}$
	\item $c::\text{y}-\text{f}$
\end{enumerate}

This grammar will generate the derived tree:

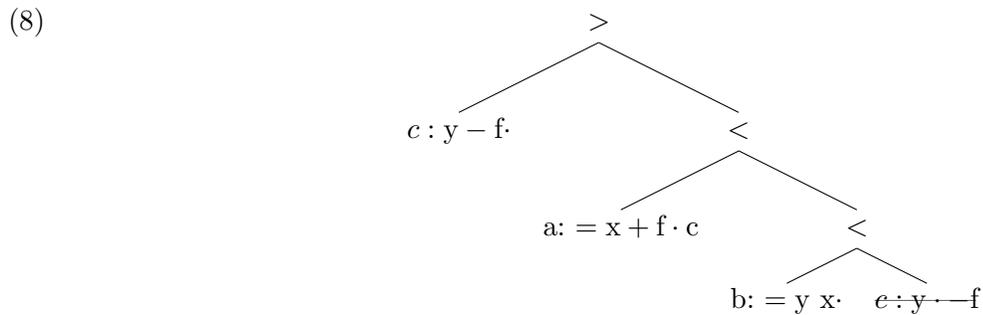
\ex.\label{toyderived}
\Tree
	[.{$>$}
		[.{$c:\text{y}-\text{f}\cdot$}
		]
		[.{$<$}
			[.{a:$\ =\text{x}+\text{f}\cdot \text{c}$}
			]
			[.{$<$}
				[.{b:$\ =\text{y}\ \text{x}\cdot$}
				]
				[.{$\rlap{-----------}{c:\text{y}\cdot-\text{f}}$}
				]
			]
		]
	]
	
Corresponding to this tree is the derivation tree:

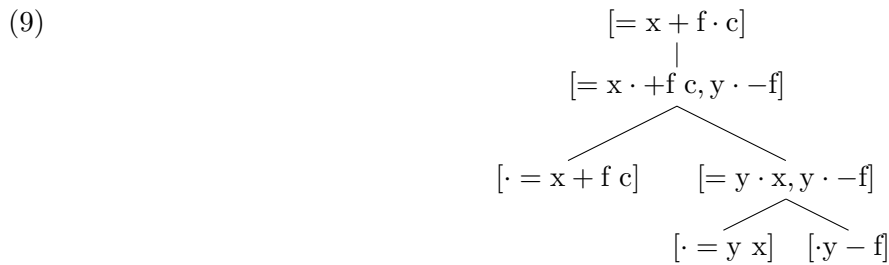
\ex.\label{toyderivation}
\Tree
	[.{$[=\text{x}+\text{f}\cdot\text{c}]$}
		[.{$[=\text{x}\cdot+\text{f}\ \text{c},\text{y}\cdot-\text{f}]$}
			[.{$[\cdot=\text{x}+\text{f}\ \text{c}]$}
			]
			[.{$[=\text{y}\cdot\text{x},\text{y}\cdot-\text{f}]$}
				[.{$[\cdot=\text{y}\ \text{x}]$}
				]
				[.{$[\cdot\text{y}-\text{f}]$}
				]
			]
		]
	]

The parser should try to expand the nodes leading to the first leaf \emph{in the derived tree} \ref{toyderived}, but is actually building the \emph{derivation tree} \ref{toyderivation}. As such, it should begin by expanding right-most nodes, then switch back to left-most ones when $c$ is parsed to parse $a$, etc... Position indices showing in which position which category is can be computed online and incorporated to the derivation tree, for example $0/$ for all categories corresponding to $c$, since its final position is just one branch down and left from the root. The parser will just have to expand the unexpanded nodes with lowest (i.e. leftmost) position indices. In order to do this, it can keep track of a \emph{pointer} telling up to which point nodes have been expanded, and expand the corresponding one. Then upgrading the pointer with the adequate notion of successor keeps the parser working. To formalize this:

\begin{defi}\label{position index}
	A \emph{position index} is a element $\pi\in\{0,1\}^*\cup\{-1\}$.
	
	Its \emph{successor} $s(\pi)$ is defined to be:
	\[s(\pi)=\left\{\begin{array}{ll}
		\alpha 1	&	\text{if $\pi=\alpha 0\beta, \beta\in1^*$}\\
		-1	&	\text{if $\pi\in1^*$}\\
		\text{undefined}	&	\text{otherwise}
\end{array}\right.
\]
Two positions indices $\pi, \pi'$ \emph{correspond} if $\pi'=\pi\beta,\beta\in0^*$. In this case, we say also that $\pi$ \emph{points to} $\pi'$.
\end{defi}

The notion of \emph{correspondence} enables the parser to have some liberty in the pointer indicating the index to be expanded. Indeed, the parser will not try to expand the node with the index exactly equal to the pointer, but just corresponding to it, that is, equal to the pointer with as many $0$s as possible following, or, in the derived tree, down the leftmost path from the node indicated by the pointer, which is what we would want: the first unexpanded node down the pointer.

\begin{defi}
	A \emph{situated category} is a pair $\alpha^n/[F^n]$, where $\alpha^n$ is a sequence of $n$ position indices and $F^n$ is a sequence of $n$ dotted feature strings (so that $[F^n]$ is a category). For readability, we will write $\langle\alpha_1,\ldots,\alpha_n\rangle/[F_1,\ldots,F_n]$ as $[\alpha_1/F_1,\ldots,\alpha_n/F_n]$.
\end{defi}

\begin{defi}
	A \emph{hypothesis} is a 5-uple $(T,\pi,p,s,\Delta)$ where:
	\begin{itemize}
		\item $T$ is a finite set of situated categories (the nodes of the partial derivation tree),
		\item $\pi$ is a position index, the \emph{pointer}, pointing to the next node to expand,
		\item $p\in[0,1]$ is the probability of the hypothesis,
		\item $s$ is a dotted input string, and
		\item $\Delta$ is the sequence of rules used to obtain this hypothesis from the axiom $start$.
	\end{itemize}
\end{defi}

The dotted input string $s$ is the string of word of the phrase being parsed, with a dot indicating up to which point it has already been parsed (in fact, up to which point the words have been scanned). For example, if $s = \text{``The cat has}\cdot\text{eaten the mouse''}$, this means that this hypothesis has already scanned (i.e., recognised a node for) the words ``The'', ``cat'' and ``has'', but not yet ``eaten'', ``the'' and ``mouse''.

\subsection{Position indices and nodes}

The parser will expand the hypothesis trees in a quite particular way, corresponding to a left-to-right reading of the output sentence. Since movement is possible in MG, the parser will have to keep track of the `position' of the different elements, to only expand the leafs corresponding to the currently parsed word. This is the role of the position indices.

\newpage
A position index different from $-1$ will represent a particular subtree in the final \emph{derived} tree, where the traces of the moved nodes are deleted (moving up its sister to the position of its mother). The position index $\alpha_0\ldots\alpha_k$ corresponds to the subtree dominated by the node obtained by going down in the tree from its root, left if $\alpha_0=0$, right otherwise, then again left if $\alpha_1=0$, right otherwise, etc...

Back to our toy grammar \ref{toygrammar}:

\begin{enumerate}
	\item $a::\ =\text{x}+\text{f}\ \text{c}$
	\item $b::\ =\text{y}\ \text{x}$
	\item $c::\text{y}-\text{f}$
\end{enumerate}

The derivation tree with indexed relevant features corresponding to \ref{toyderived}:

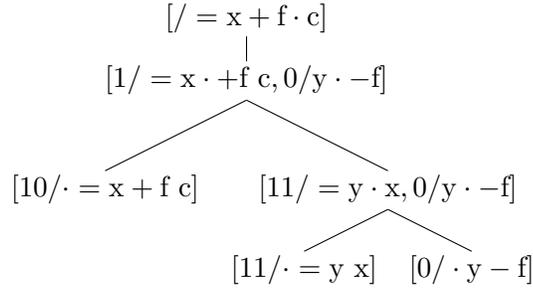
\Tree
	[.{$[/=\text{x}+\text{f}\cdot\text{c}]$}
		[.{$[1/=\text{x}\cdot+\text{f}\ \text{c},0/\text{y}\cdot-\text{f}]$}
			[.{$[10/\cdot=\text{x}+\text{f}\ \text{c}]$}
			]
			[.{$[11/=\text{y}\cdot\text{x},0/\text{y}\cdot-\text{f}]$}
				[.{$[11/\cdot=\text{y}\ \text{x}]$}
				]
				[.{$[0/\cdot\text{y}-\text{f}]$}
				]
			]
		]
	]

The indexed relevant category at the root of the derivation tree has a empty position string since it represent the derived tree itself, and in $[11/=\text{y}\cdot\text{x},0/\text{y}\cdot-\text{f}]$ for example, we have $11/=\text{y}\cdot\text{x}$ because this relevant category represent the tree under the node obtained if you go right ($1$), then right again ($11$) from the root node of the derived tree (without the moving categories, since they will move so won't end up at the same place). We have similarly $0/\text{y}\cdot-\text{f}$ because the subtree described by $\text{y}\cdot-\text{f}$ ends up as the left ($0$) daughter of the root of the derived tree.

The assignment of these position strings is given by the inference rules, which will be discussed later.

\subsection{Axiom}

The axiom of the parser are exactly the same as the axiom for the LCFRS corresponding to our MG discussed in \ref{PMG}, plus an empty position string (it represents the whole derived tree...). Its probability will be of course 1, and the pointer will be set as $\epsilon$. So, if the phrase to be parsed is $\omega$, we have a parser axiom $(\epsilon/start,\epsilon,1,\cdot\omega,\langle\ \rangle)$.

\subsection{Inference rules}

We here have exactly the same inference rules as before, exept that these will assign position strings too. So here they are:

		\begin{enumerate}
			\item Start rules: for every lexical item $\gamma::\delta\ \text{c}$,\\
				\textbf{Start: }
				$\epsilon/start\longrightarrow[\epsilon/\delta\cdot\text{c}]$\newpage
			\item \label{un-merge} Un-merge rules: the left-hand category is of the form $[\alpha/\delta =\text{x} \cdot \theta, S]$
				\begin{enumerate}
					\item	\label{simplerule} Cases where the selector was a simple tree $(\delta=\epsilon)$:
						\begin{enumerate}
							\item \label{un-merge-1} For any lexical item of feature string $\gamma\ \text{x}$,\\
								\textbf{Unmerge-1:}
							
								\begin{tabular}{cc}
												$[\alpha/=\text{x} \cdot \theta, S]\longrightarrow$ &
												\tab\Tree [.{$\bullet$}
																[.{$[\alpha0/\cdot =\text{x}\ \theta]$}
																]
																[.{$[\alpha1/\gamma \cdot \text{x}, S]$}
																]
															]\etab
								\end{tabular}

								$t$ is here $s$ if $\gamma=\emptyset$ (and thus $S=\emptyset$ too), and $c$ otherwise.
							\item \label{un-merge-3, simple} For any element $(\gamma\ \text{x}\cdot \varphi) \in S$, with $S'=S-(\gamma\ \text{x}\cdot \varphi)$,\\
								\textbf{Unmerge-3, simple:}
								
								\begin{tabular}{cc}
									$[\alpha/=\text{x} \cdot \theta,\beta/\gamma\ \text{x}\cdot \varphi, S']\longrightarrow$	&	
									\tab\Tree
										[.{$\bullet$}
											[.{$[\alpha/\cdot =\text{x}\ \theta]$}
											]
											[.{$[\beta/\gamma \cdot \text{x}\ \varphi, S']$}
											]
										]\etab 
								\end{tabular}

								$t$ is here $s$ if $\gamma=\emptyset$ (and thus $S'=\emptyset$ too), and $c$ otherwise. It should be noted that necessarily, $\varphi\neq\emptyset$.
						\end{enumerate}
					\item \label{complexrule} Cases where the selector was a complex tree:
						\begin{enumerate}
							\item \label{un-merge-2} For any decomposition $S=U\sqcup V$, and any lexical item of feature string $\gamma\ \text{x}$,\\
								\textbf{Unmerge-2:}
								
								\begin{tabular}{cc}
									$[\alpha/\delta=\text{x} \cdot \theta, S]\longrightarrow$	&
									\tab\Tree
										[.{$\bullet$}
											[.{$[\alpha1/\delta\cdot =\text{x}\ \theta, U]$}
											]
											[.{$[\alpha0/\gamma \cdot \text{x}, V]$}
											]
										]\etab
								\end{tabular}
								  
								$t$ is, as always, $s$ if $\gamma=\emptyset$ (and thus $V$ has to be empty too), and $c$ otherwise.
							\item \label{un-merge-3, complex} For any element $(\gamma\ \text{x}\cdot \varphi) \in S$, and any decomposition $S=U\sqcup V\sqcup (\gamma\ \text{x}\cdot \varphi)$,\\
								\textbf{Unmerge-3, complex:}
								
								\begin{tabular}{cc}
									$[\alpha/\delta=\text{x} \cdot \theta,\beta/\gamma\ \text{x}\cdot \varphi, S']\longrightarrow$	&
									\tab\Tree
										[.{$\bullet$}
											[.{$[\alpha/\delta\cdot =\text{x}\ \theta, U]$}
											]
											[.{$[\beta/\gamma \cdot \text{x}\ \varphi, V]$}
											]
										]\etab
								\end{tabular}
								 
								$t$ is still $s$ if $\gamma=\emptyset$ (and thus $V$ has to be empty too), and $c$ otherwise. As in \ref{Unmerge-3, simple}, $\varphi$ has to be non empty.
						\end{enumerate}
				\end{enumerate}
			\item \label{moverule} Un-move rules: the left-hand category is of the form $[\delta +\text{f} \cdot \theta, S]$
				\begin{enumerate}
					\item \label{un-move-2} For any $(\gamma-\text{f}\cdot\varphi)\in S$ (necessarily unique by the Shortest Movement Constraint), with $S'=S-(\gamma -\text{f}\cdot \varphi)$,\\
						\textbf{Unmove-2:}
						
						\begin{tabular}{cc}
							$[\alpha/\delta'+\text{f} \cdot \theta,\beta/\gamma-\text{f}\cdot\varphi, S']\longrightarrow$	&
							\tab\Tree
										[.{$\circ$}
											[.{$[\alpha/\delta'\cdot +\text{f}\ \theta,\beta/\gamma\cdot-\text{f}\ \varphi, S']$}
											]
										]\etab
						\end{tabular}
						 
					\item \label{un-move-1} If there is no $(\gamma-\text{f}\cdot\varphi)\in S$, then for any lexical item of feature string $\gamma -\text{f}$,\\
						\textbf{Unmove-1:}
						
						\begin{tabular}{cc}
							$[\alpha/\delta'+\text{f} \cdot \theta, S]\longrightarrow$	&
							\tab\Tree
										[.{$\circ$}
											[.{$[\alpha1/\delta'\cdot +\text{f}\ \theta, \alpha0/\gamma\cdot-\text{f}, S]$}
											]
										]\etab
						\end{tabular}
						 
				\end{enumerate}
		\end{enumerate}

There is no lexicalise rule, since it will in fact be replaced by a `\emph{scan} rule', checking if the feature string of the word currently parsed corresponds to the current feature string.

\subsection{Top-down parser}
\label{TDP}

The parser takes an \emph{input string} $\omega=\omega_0\ldots\omega_{n-1}$, a minimalist grammar $\mathcal{G}$ rewrited into a LCFRS $\mathcal{S}$ and a \emph{beam function} $f$, setting a threshold to the probability of the selected hypothesis. It will work on a \emph{priority queue} of hypothesis $\mathcal{H}$. The function $f$ can be very general, here we will consider that its argument is the priority queue $\mathcal{H}$. The parser works as following:

\begin{enumerate}
	\item \textbf{Beginning:} The parser start with the queue of hypothesis consisting of the axiom $(\epsilon/start,\epsilon,1,\cdot\omega,\langle\ \rangle)$ of the grammar.
	\item \textbf{Expanding:} At each step, the parser will:
		\begin{itemize}
			\item take the top-ranked hypothesis $(T,\pi,p,\omega_0^k\cdot\omega_k^n,\Delta)$ (i.e. the hypothesis with greatest $p$) in the priority queue,
			\item check the corresponding position string pointer. If $\pi=-1$, and the parsing dot in $\omega_0^k\cdot\omega_k^n$ is at the far right (i.e. $k=n$), then the parser terminates and returns the sequence of rules $\Delta$. If the phrase is not completely parsed ($\pi=-1$ but $k<n$), the hypothesis is deleted and the parser moves to the next one. If $\pi\neq-1$, the parser moves to the next step, and tries to:
			\item find the leaf of $T$, $C$, in which is the position string $\alpha$ corresponding to the pointer $\pi$. If $\alpha\neq\pi$, $\pi$ is set to $\alpha$.
			\item expand $C$. For this, we have two possibilities:
				\begin{enumerate}
					\item \textbf{Expand:} If $C$ is a complex situated category, the parser will delete the current hypothesis $(T, p, \pi, \omega_0^k\cdot\omega_k^n,\Delta)$ and add to the priority queue, for all possible inference rules $C\longrightarrow t$, a new hypothesis $(T', \pi', p\mathbb{P}(C\longrightarrow t), \omega_0^k\cdot\omega_k^n,\Delta @ C\longrightarrow t)$, such that $T'$ is $T$ where the node $C$ has been replaced by $t$, and $\pi'$ is either $\pi0$ if the rule did change the value of the position string corresponding to $\pi$ (i.e. if the rule was Unmerge-1, Unmerge-2 and Unmove-1, and the first element of $C$ had $\alpha$ for position string), and $\pi$ in the other cases. $@$ is the concatenation operator.
					\item \textbf{Scan:} If $C$ is a simple indexed category, say $C=[\alpha/\cdot\delta]$, then the parser will delete the current hypothesis $(T, \pi, p, \omega_0^k\cdot\omega_k^n,\Delta)$, and try to lexicalize $C$. It will do two things:
						\begin{enumerate}
							\item \textbf{Scan, $\epsilon$:} If there is a rule $[\cdot\delta]\longrightarrow \epsilon::\delta$, then a new hypothesis $(T', s(\pi), p\mathbb{P}([\cdot\delta]\longrightarrow \epsilon::\delta), \omega_0^k\cdot\omega_k^n,\Delta @[\cdot\delta]\longrightarrow \epsilon::\delta)$ is added to the priority queue, where $T'$ is $T$ where the leaf $C$ was replaced by $\epsilon::\delta$.
							\item \textbf{Scan, $\rlap{/}\epsilon$:} If $\omega_k::\delta$ is in the grammar, then a new hypothesis $(T', s(\pi), p\mathbb{P}([\cdot\delta]\longrightarrow \omega_k::\delta), \omega_0^k\omega_k\cdot\omega_{k+1}^n,\Delta @[\cdot\delta]\longrightarrow \omega_k::\delta)$ is added to the priority queue, where $T'$ is $T$ where the leaf $C$ was replaced by $\omega_k::\delta$.
						\end{enumerate}
							If these two steps fail, then no hypothesis is added to the priority queue.
				\end{enumerate}
				The new hypothesis are inserted in the priority queue at their `right place', i.e. after all hypothesis of higher probability.
			\item \textbf{Prune:} The parser deletes all hypothesis of the priority queue whose probability is lower than $f(\mathcal{H})$.
		\end{itemize}
	
	If $\mathcal{H}$ is empty, then the parse failed and the sentence is judged ungrammatical.
\end{enumerate}

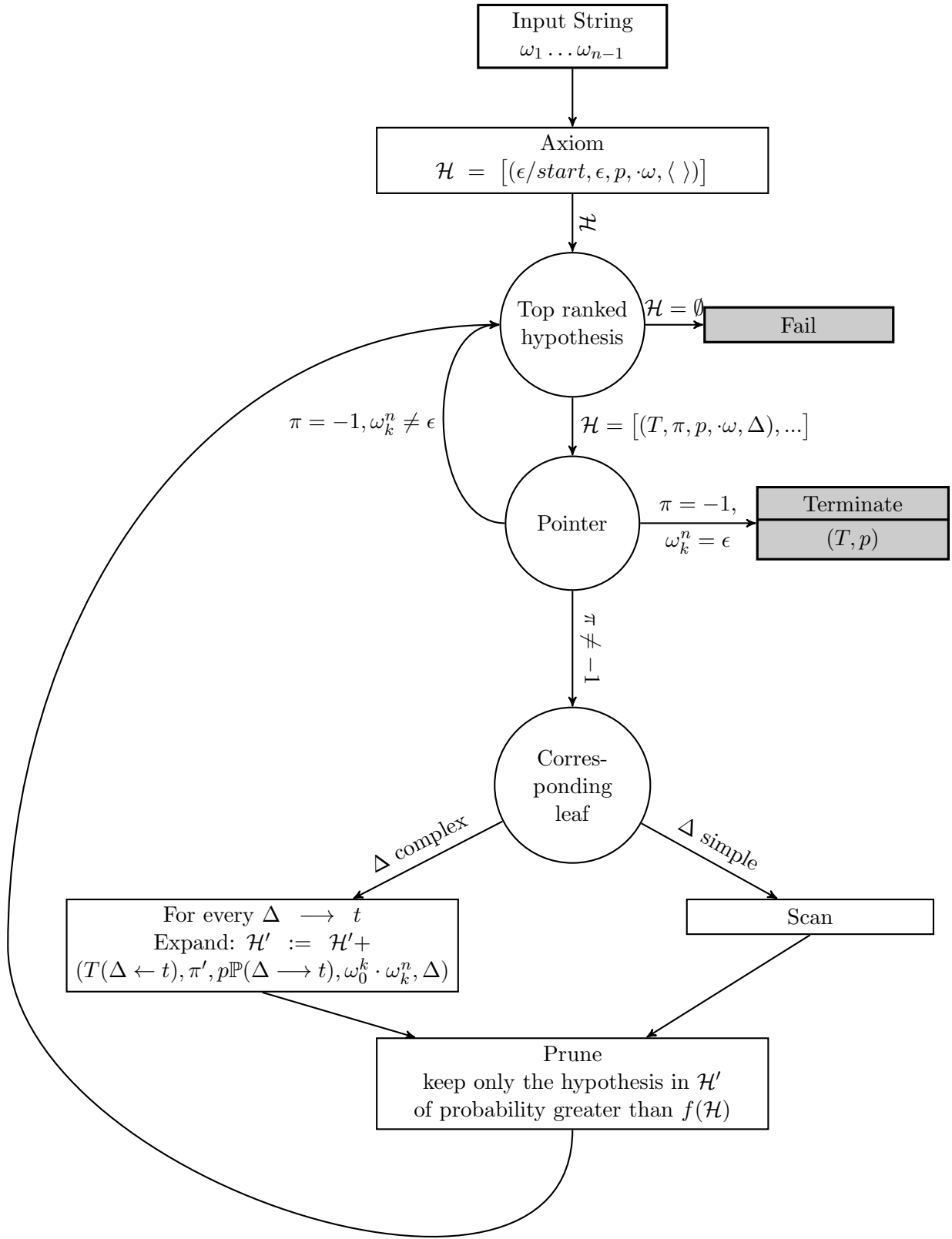
\begin{center}
\begin{tikzpicture}[
	initial/.style={
		% The shape:
		rectangle,
		% The size:
		minimum size=6mm,
		% The border:
		very thick,draw=black!100, text width=3cm, text centered},
	terminal/.style={
		rectangle split, rectangle split parts=2,
		minimum size=6mm,
		very thick,draw=black!100,
		fill=black!20, text width=3cm, text centered},
	fail/.style={
		rectangle,
		minimum size=6mm,
		very thick,draw=black!100,
		fill=black!20, text width=3cm, text centered},
	action/.style={
		rectangle,
		minimum size=6mm,
		draw=black!100, text width=6.5cm, text centered},
	action2/.style={
		rectangle,
		minimum size=6mm,
		draw=black!100, text width=4cm, text centered},
	sample/.style={
		circle,
		minimum size=6mm,
		draw=black!100,text width=2cm, text centered},
	point/.style={
		rectangle,
		minimum size=6mm,
		draw=black!100,text width=2.3cm, text centered},
	point/.style={coordinate},>=stealth',thick,draw,every join/.style={->},skip loop/.style={to path={-- ++(#1,0) -| (\tikztotarget)}}
	]
	\node (init) [initial] {Input String \\
														$\omega_1\ldots\omega_{n-1}$};
	\node (axiom)	[action, below=of init] {Axiom\\
																				$\mathcal{H}=\big[(\epsilon/start,\epsilon,p,\cdot\omega,\langle\ \rangle)\big]$};
	\node (first) [sample, below=of axiom]	{Top ranked hypothesis};
	\node (fail) [fail, right=of first] {Fail};
	\node (pointer) [sample, below=of first]	{Pointer};
	\node (null2) [point, right=of pointer] {};
	\node (null3) [point, below=of pointer] {};
	\node (term) [terminal, right=of null2]	{Terminate
																							\nodepart{second}
																							$(T,p)$};
	\node (leaf) [sample, below=of null3]	{Corres- ponding leaf};
	\node (expand) [action, below left=of leaf]	{For every $\Delta\longrightarrow t$\\
																								Expand: $\mathcal{H'}:=\mathcal{H}'+$\\
																								$(T{(\Delta\leftarrow t)}, \pi', p\mathbb{P}{(\Delta\longrightarrow t)}, \omega_0^k\cdot\omega_k^n,\Delta)$};
	\node (int0) [point, left=of expand] {};
	\node (scan) [action2, below right=of leaf]	{
																							Scan};
	\node (null) [point, right=of scan] {};
	\node (int1) [point, below=of leaf] {};
	\node (int2) [point, below=of int1] {};
	\node (prune) [action, below=of int2] {Prune\\keep only the hypothesis in $\mathcal{H}'$ of probability greater than $f(\mathcal{H})$};
	\draw [->] (init) to (axiom);
	\draw [->] (axiom) to node[above,sloped] (a) {$\mathcal{H}$} (first);
	\draw [->] (first) to node[right] (b) {$\mathcal{H}=\big[(T,\pi,p,\cdot\omega,\Delta),...\big]$} (pointer);
	\draw [->]	(pointer) to node[above,sloped] {$\pi=-1,$} node[below,sloped] {$\omega_k^n=\epsilon$} (term);
	\draw [->] (pointer) to node[above,sloped] (d) {$\pi\neq-1$} (leaf);
	\draw [->] (leaf) to node[above,sloped] (e) {$\Delta$ complex} (expand);
	\draw [->] (leaf) to node[above,sloped] (f) {$\Delta$ simple} (scan);
	\draw [->] (expand.south) to (prune);
	\draw [->] (scan.south) to (prune);
	\draw [->] (pointer.west) to [out=180,in=180] node[left] (g) {$\pi=-1, \omega_k^n\neq\epsilon$} (first.west);
	\draw [-] (prune.south) to [out=270,in=270](int0);
	\draw [->] (int0) to [out=90,in=180](first.west);
	\draw [->] (first) to node[above,sloped] (f) {$\mathcal{H}=\emptyset$} (fail);
\end{tikzpicture}
\end{center}

\subsection{Example}

Here we will present a small example of the parsing of a particular sentence of the grammar we presented in \ref{anbn}, consisting of the the lexical items:

\begin{enumerate}
	\item $\epsilon::\text{c}$
	\item $\epsilon::\ =\text{a}+\text{m}\ \text{c}$
	\item $a::\ =\text{b}\ \text{a}-\text{m}$
	\item $b::\text{b}$
	\item $b::\ =\text{a}+\text{m}\ \text{b}$
\end{enumerate}

The corresponding LCFRS was consisting of these rules:

\qtreecenterfalse

{\center
\begin{tabular}{|lrllcl|}
	\hline
	S1	&$start$	&$\longrightarrow$	&$[\cdot\text{c}]$	&Start	&$\mathbb{P}(.)=.7$	\\
	S2	&$start$	&$\longrightarrow$	&$[=\text{a}+\text{m}\cdot\text{c}]$	&Start	&$\mathbb{P}(.)=.3$	\\
	\hline
	L1	&$[\cdot\text{c}]$	&$\longrightarrow$	&$\epsilon::\text{c}$	&Lexicalize	&$\mathbb{P}(.)=1$	\\
	\hline
	Mv1	&$[=\text{a}+\text{m}\cdot\text{c}]$	&$\longrightarrow$
		&\tab\Tree
			[.{$\circ$}
				[.{$[=\text{a}\cdot+\text{m}\ \text{c},\text{b}\ \text{a}\cdot-\text{m}]$}
				]
			]\etab
		&Unmove-1	&$\mathbb{P}(.)=1$	\\
	\hline
	Mg1	&$[=\text{a}\cdot+\text{m}\ \text{c},\text{b}\ \text{a}\cdot-\text{m}]$	&$\longrightarrow$
		&\tab\Tree
			[.{$\bullet$}
				[.{$[\cdot=\text{a}+\text{m}\ \text{c}]$}
				]
				[.{$[=\text{b}\cdot\text{a}-\text{m}]$}
				]
			]\etab
		&Unmerge-3, simple	&$\mathbb{P}(.)=1$	\\
	\hline
	L2	&$[\cdot=\text{a}+\text{m}\ \text{c}]$	&$\longrightarrow$	&$\epsilon::\ =\text{a}+\text{m}\ \text{c}$	&Lexicalize	&$\mathbb{P}(.)=1$	\\
	\hline
	Mg2	&$[=\text{b}\cdot\text{a}-\text{m}]$	&$\longrightarrow$
		&\tab\Tree
			[.{$\bullet$}
				[.{$[\cdot=\text{b}\ \text{a}-\text{m}]$}
				]
				[.{$[\cdot\text{b}]$}
				]
			]\etab
		&Unmerge-1, simple	&$\mathbb{P}(.)=.4$	\\
	Mg3	&$[=\text{b}\cdot\text{a}-\text{m}]$	&$\longrightarrow$
		&\tab\Tree
			[.{$\bullet$}
				[.{$[\cdot=\text{b}\ \text{a}-\text{m}]$}
				]
				[.{$[=\text{a}+\text{m}\cdot\text{b}]$}
				]
			]\etab
		&Unmerge-1, simple	&$\mathbb{P}(.)=.6$	\\
	\hline
	L3	&$[\cdot=\text{b}\ \text{a}-\text{m}]$	&$\longrightarrow$	&$ a::\ =\text{b}\ \text{a}-\text{m}$	&Lexicalize	&$\mathbb{P}(.)=1$	\\
	\hline
	L4	&$[\cdot\text{b}]$	&$\longrightarrow$	&$ b::\text{b}$	&Lexicalize	&$\mathbb{P}(.)=1$	\\
	\hline
	Mv2	&$[=\text{a}+\text{m}\cdot\text{b}]$	&$\longrightarrow$
		&\tab\Tree
			[.{$\circ$}
				[.{$[=\text{a}\cdot+\text{m}\ \text{b},=\text{b}\ \text{a}\cdot-\text{m}]$}
				]
			]\etab
		&Unmove-1	&$\mathbb{P}(.)=1$	\\
	\hline
	Mg4	&$[=\text{a}\cdot+\text{m}\ \text{b},=\text{b}\ \text{a}\cdot-\text{m}]$	&$\longrightarrow$
		&\tab\Tree
			[.{$\bullet$}
				[.{$[\cdot=\text{a}+\text{m}\ \text{b}]$}
				]
				[.{$[=\text{b}\cdot\text{a}-\text{m}]$}
				]
			]\etab
		&Unmerge-3, simple	&$\mathbb{P}(.)=1$	\\
	\hline
	L5	&$[\cdot=\text{a}+\text{m}\ \text{b}]$	&$\longrightarrow$	&$ b::\ =\text{a}+\text{m}\ \text{b}$	&Lexicalize	&$\mathbb{P}(.)=1$	\\
	\hline
\end{tabular}}

Here we took $\lambda=.7$ and $\mu=.4$.

We will now try to parse the string $aabb$, which is generated by the grammar (the $\epsilon$ is of course omitted).

\newpage
\begin{itemize}
	\item The parser begins with his stack consisting of the \emph{axiom} of the grammar:
		\[\mathcal{H}=\left((\epsilon/start,\epsilon,1,\cdot aabb,\langle\ \rangle)\right)
	\]
	\item The parsers takes the top-ranked hypothesis, $(\epsilon/start,\epsilon,1,\cdot aabb,\langle\ \rangle)$, its pointer (null), the corresponding leaf, $start$, and tries to expand it. There are two possibilities, which are added to the hypothesis queue, ordered by decreasing possibilities:
		\[\mathcal{H}=\left(([\epsilon/\cdot \text{c}],\epsilon, .7,\cdot aabb,\langle S1 \rangle),([/=\text{a}+\text{m}\cdot \text{c}],\epsilon, .3,\cdot aabb,\langle S2 \rangle )\right)
	\]
	\item The parser now takes the top-ranked hypothesis, $([\epsilon/\cdot \text{c}],\epsilon, .7,\cdot aabb,\langle S1 \rangle)$, its pointer (null), the corresponding leaf, $[/\cdot \text{c}]$, and try to scan it since it is a simple category. There is only one rule whose left-size is $[/\cdot \text{c}]$, L1, with probability $1$. The corresponding word is empty, so the scan succeeds, the pointer is increased to $-1$ and the new hypothesis is added to the queue in place of the old one:
		\[\mathcal{H}=\left(([\epsilon/\epsilon::\text{c}],-1, .7,\cdot aabb,\langle S1,L1 \rangle),([/=\text{a}+\text{m}\cdot \text{c}],\epsilon, .3,\cdot aabb,\langle S2 \rangle )\right)
	\]
	\item The parser takes once more the top-ranked analysis, $([\epsilon/\epsilon::\text{c}],-1, .7,\cdot aabb,\langle S1,L1 \rangle)$. Its pointer is $-1$, so the parser checks if the parse is indeed over. No, since the string left of the dot is non-empty. The current hypothesis is now deleted, and the new queue is fed to the parser:
		\[\mathcal{H}=\left(([/=\text{a}+\text{m}\cdot \text{c}],\epsilon, .3,\cdot aabb,\langle S2 \rangle )\right)
	\]
	\item The top-ranked analysis is now the only one in the queue, $(\cdot aabb, [/=\text{a}+\text{m}\cdot \text{c}], .3,\ )$. Its pointer is null, the corresponding leaf is $[/=\text{a}+\text{m}\cdot \text{c}]$, which the parser will try to expand. There is only one rule, Mv1:$[=\text{a}+\text{m}\cdot\text{c}]\longrightarrow$
		\tab\Tree
			[.{$\circ$}
				[.{$[=\text{a}\cdot+\text{m}\ \text{c},\text{b}\ \text{a}\cdot-\text{m}]$}
				]
			]\etab
		, so the hypothesis is replaced by the new one:
		\[\mathcal{H}=\left((\text{
		\tab\Tree
			[.{$\circ$}
				[.{$[1/=\text{a}\cdot+\text{m}\ \text{c},0/\text{b}\ \text{a}\cdot-\text{m}]$}
				]
			]\etab
			},0,.3,\cdot aabb,\langle S2,Mv1 \rangle )\right)
	\]
	(note that the pointer was modified since the position vector of the expanded element was modified by the rule)
	\item The parser goes on, giving the new queue:
	\[\mathcal{H}=\left((\text{
		\tab\Tree
			[.{$\circ$}
				[.{$\bullet$}
					[.{$[1/\cdot=\text{a}+\text{m}\ \text{c}]$}
					]
					[.{$[0/=\text{b}\cdot\text{a}-\text{m}]$}
					]
				]
			]\etab
			},0, .3,\cdot aabb, \langle S2,Mv1,Mg1 \rangle)\right)
	\]
	\item And so forth...
\end{itemize}

\newpage

\section{Proofs of soundness and completeness}

\subsection{Soundess of the pointer}

We will here demonstrate that the parser, at each step, will indeed find the correct leaf to expand with the current pointer.

First we modify a little, for convenience of the proof, our definition of a position string:
\begin{defi}
	A \emph{position index} (and a \emph{pointer}) is a dotted, almost null binary sequence $\alpha\cdot\beta, \alpha\beta\in\{0,1\}^k\bar{0}$ for some $k$.
	
	A position index and a pointer \emph{correspond} if their undotted sequences are the same.
	
	The set of all undotted position indexes is naturally ordered by the lexicographic order.
\end{defi}

Note: This definition is consistent with the one I proposed in the precedent section, the dot being the place where the `new' position indexes are to be truncated to obtain the `old' ones.

\begin{prop}
	There is a one-to-one application from the set of all position indexes $\mathcal{I}$ to the set $\mathcal{ST}$ of all subtrees of the infinite complete binary tree $\mathcal{T}$. The head of the subtrees corresponding to the positions indexes of the type $\alpha_1\ldots\alpha_n\cdot\bar{0}$ are exactly the nodes of depth $n$ of the tree. We thus have a notion of domination on the position indexes, corresponding to the notion of domination in the tree (by convention, a node dominates itself). A position index $\alpha\cdot\bar{0}$ dominate another position index $\alpha'\cdot\bar{0}$ if $\alpha$ is a prefix of $\alpha'$.
\end{prop}

\begin{proof}
	Let $\phi : $\begin{tabular}{ccc}
	$\mathcal{I}$&$\longrightarrow$&$\mathcal{ST}$\\
	$\alpha_1\ldots\alpha_n\cdot\bar{0}$&$\longrightarrow$&$T$
\end{tabular} where $T$ is the subtree headed by the node obtained by, starting from the root of $\mathcal{T}$, for each $\alpha_i$, going left if $\alpha_i=0$ and right otherwise.
This application has clearly all the above properties.
\end{proof}

\begin{lemme}
	For every cut $C$ of position indexes,
	\begin{itemize}
	\item[\phantom{ii}i.]\label{i} if $\beta\cdot\bar{0}$ is in the cut $C$, then there is no $\beta\gamma\cdot\bar{0}$ in $C$, for every $\gamma\in\{0,1\}^+$.
	\item[\phantom{i}ii.] If $\beta0\gamma\cdot\bar{0}$ is in the cut $C$, then there is a unique $k\in\mathbb{N}$ such that $\beta10^k\cdot\bar{0}$ is in the cut.
	\item[iii.] The lexicographic order can be extended to the \emph{dotted} elements of $C$.
\end{itemize}
\end{lemme}

\begin{proof}
	Let $n$ be the depth of the cut.
	\begin{itemize}
	\item[\phantom{ii}i.] Suppose that we have $\beta\cdot\bar{0}$ and $\beta\gamma\cdot\bar{0}$ in the cut, for some $\gamma\in\{0,1\}^+$. Then $\beta\gamma0^n\cdot\bar{0}$ is dominated by both $\beta\cdot\bar{0}$ and $\beta\gamma\cdot\bar{0}$, which is a contradiction.
	\item[\phantom{i}ii.] Since $C$ is a cut, $\beta10^n\cdot\bar{0}$ must be dominated by a unique element of $C$, say $\delta\cdot\bar{0}$. $\delta$ is then a prefix of $\beta10^n$. But $\delta$ cannot be a prefix of $\beta0\gamma$, since otherwise $\delta\cdot\bar{0}$ would dominate $\beta0\gamma\cdot\bar{0}$, which is already dominated by itself. So $\delta$ must be of the form $\beta10^k, k<n$. The unicity follows from i.
	\item[iii.] This follows directly from i.
\end{itemize}
\end{proof}

We can now prove that the parser will never have a pointer problem:

\begin{theo}
	At each point of the parse, the position indexes of (all) the hypothesis form a cut, the already scanned position indexes form a prefix set of all the position indexes (for the lexicographic order), and the pointer correspond to the smallest unscanned position index (which exists since the set is finite), or is possibly -1 if there is none.
\end{theo}

\begin{proof}
	We'll prove this result by recurrence on the number of steps done by the parser.
\begin{itemize}
	\item At the beginning, the set of the position indexes is reduced to $\{\cdot\bar{0}\}$, and the pointer is $\cdot\bar{0}$, so the result is trivially true.
	\item Suppose that at step $n$, the position indexes of (all) the hypothesis form a cut, the already scanned position indexes form a prefix set of all the position indexes (for the lexicographic order), and the pointer correspond to the smallest scanned position index (which exists since the set is finite), or is possibly -1 if there is none. Let $\alpha\cdot\bar{0}$ be the position index corresponding to the pointer (unique by hypothesis), and $[\beta\cdot\bar{0}/\Delta,\ldots,\alpha\cdot\bar{0}/\Delta',\ldots]$ the leaf to be expanded. By hypothesis, $\alpha\cdot\bar{0}\leq\beta\cdot\bar{0}$. Let $(\delta_1,\ldots,\delta_k//\alpha\cdot\bar{0},\ldots,\beta\cdot\bar{0},\ldots)$ the positions indexes, lexicographically ordered, the scanned ones left of //. By hypothesis, this is a cut. The state of the parser at step $n+1$ can be obtained by four different cases:
	\begin{enumerate}
		\item if the position index corresponding to the pointer, $\alpha\cdot\bar{0}(<\beta\cdot\bar{0})$, is not modified by the rules, two cases:
			\begin{enumerate}
				\item if the main position index of the leaf being expanded, $\beta\cdot\bar{0}$, is not modified (i.e. during Un-merge-3 or Un-move-2), no position indexes are modified, nor the pointer, so the result still holds.
				\item if the main position index of the leaf being expanded, $\beta\cdot\bar{0}$, is modified (i.e. during Un-merge-1,-2 or Un-move-1), the new position indexes are (lexicographically ordered, the already scanned ones left of the //)
				
				$(\delta_1,\ldots,\delta_k//\alpha\cdot\bar{0},\ldots,\beta0\cdot\bar{0},\beta1\cdot\bar{0},\ldots)$.
				
				This is trivially still a cut, the pointer is still $\alpha\cdot\bar{0}$, corresponding to the smallest unscanned position indexes $\alpha\cdot\bar{0}$.
			\end{enumerate}
		\item if the position index corresponding to the pointer, $\alpha\cdot\bar{0}(=\beta\cdot\bar{0})$, is modified by the rules, the new position indexes are
		
			$(\delta_1,\ldots,\delta_k//\alpha0\cdot\bar{0},\alpha1\cdot\bar{0},\ldots)$.
			
			This is still trivially a cut, and the new pointer is $\alpha0\cdot\bar{0}$, corresponding to the smallest unscanned position index $\alpha0\cdot\bar{0}$.
		\item if the parser \emph{scans} the leaf (and then $\alpha\cdot\bar{0}=\beta\cdot\bar{0}$):
		
		If $\alpha\in1^*$, the pointer is set to $-1$, there cannot be a greater item in the cut than $\alpha\in1^*$ (since it would then have to have $\alpha$ as a prefix, which is impossible by lemma \ref{i}), so the result is true.
		
		Let's then write $\alpha=\gamma01^m$. The pointer is set to $\gamma1\cdot\bar{0}$. The new position indexes are
		
		$(\delta_1,\ldots,\delta_k,\gamma01^m\cdot\bar{0}//\zeta\cdot\bar{0},\ldots)$.
		
		Since the only thing that changed here is the position of the //, this is still a cut. By lemma \ref{i}, there exists a unique $\xi=\gamma10^r$ in the cut, since $\alpha=\gamma01^m$ was in it. This being the smallest position index greater than $\alpha=\gamma01^m$, we have $\zeta=\xi=\gamma10^r$, which is corresponding to the new pointer. This ends the demonstration.
	\end{enumerate}
\end{itemize}
\end{proof}

\subsection{Completeness}

Here we demonstrate the completeness of the parser, without the pruning step, i.e. that if the string to be parsed can indeed be generated by the grammar, then the parser will eventually parse it.

\begin{lemme} \label{prefix sequence}
	Let $(a_n)_{n\in\N}\in(\Sigma^*)^{\N}$ a infinite sequence of finite distinct strings over a finite alphabet $\Sigma$. Suppose that the following hold:
\begin{align}\label{prefix}
	\forall n\in\N\text{, all prefixes of }a_n\text{ are in }\{a_k, k\leq n\}.
\end{align}
	Then there exists an infinite sequence $(x_k)_{k\in\N}\in\Sigma^\N$ such that $(x_0\ldots x_k)_{k\in\N}$ is a subsequence of $(a_n)_{n\in\N}$.
\end{lemme}

\begin{proof}
	Let's build this sequence recursively:
	\begin{itemize}
		\item Among all elements of $\Sigma$, there is an element which is prefix of infinitely many elements of $(a_n)_{n\in\N}$. Let's call it $x_0$. By hypothesis, $x_0\in(a_n)_{n\in\N}$.
		\item Suppose we have $x_0,\ldots,x_N$ such that $(x_0\ldots x_k)_{k\in\seg{0,N}}$ is a finite subsequence of $(a_n)_{n\in\N}$. Without loss of generality, we can restrict $(a_n)_{n\in\N}$ to its subsequence composed of the elements $(x_0\ldots x_k),\ k\in\seg{0,N}$ and all elements of $(a_n)_{n\in\N}$ with prefix $x_0\ldots x_N$. This new sequence is still infinite, and has the prefix property \ref{prefix}.
		
		Among all elements of $\Sigma$, there is an element $x$ such that $x_0\ldots x_Nx$ is prefix of infinitely many elements of $(a_n)_{n\in\N}$. Let $x=x_{N+1}$. By hypothesis, $x_{N+1}\in(a_n)_{n\in\N}$.
		\item $(x_k)_{k\in\N}$ has the property we seek.
	\end{itemize}
\end{proof}

\begin{theo}\label{finiteness of PDTs}
	For all $p\in(0,1]$, if there is no looping chain of rules of probability $1$ in the grammar, there are finitely many partial derivation trees (PDT) of probability $\geq p$.
\end{theo}

\begin{proof}
	A partial derivation tree is exactly defined by the string of rules deriving it. So a PDT will here be seen, when convenient, as a string of rules.
	
	Suppose there were infinitely many PDT of probability $\geq p$. Then we have a sequence of strings of rules as defined by lemma \ref{prefix sequence}. The lemma then gives us a sequence $A_0,\ldots,A_n,\ldots$ of rules such that $\forall n\ A_0 \ldots A_n$ defines a correct partial derivation tree (since $A_0 \ldots A_n$ is in the initial sequence, composed of correct PDTs).
	
	To this sequence of rules corresponds an infinite sequence of growing PDT, all of which have probability $\geq p$. Since the sequence is growing and infinite, there is an infinite path in the limit of the trees, given by the sequence of rules $(A_{\varphi(n)})_{n\in\N}$.
	
	Or, differently put, there is a sequence of rules $(A_{\varphi(n)})_{n\in\N}$, such that for all $n\geq1$, the left side of $A_n$ is in the right side of $A_{n-1}$.
	
	Then there is a finite sequence $A_{\varphi(k)},\ldots,A_{\varphi(k')}$ such that
	\begin{align}\label{loop}
		(A_{\varphi(n)})_{n\in\N}=A_{\varphi(0)},A_{\varphi(1)},\ldots,A_{\varphi(k-1)},(A_{\varphi(k)},\ldots,A_{\varphi(k')})^\N.
	\end{align}
	
	Indeed, let $\mathcal{A}$ be the finite state automaton with:
	\begin{itemize}
	\item states $A_{\varphi(n)}, n\in\N$ and $END$,
	\item starting state $A_{\varphi(0)}$, ending state $END$,
	\item transitions rules $A_{\varphi(n)}\longrightarrow A_{\varphi(n+1)}$ and $A_{\varphi(n)}\longrightarrow END$ for all n.
\end{itemize}

This automaton generates exactly all prefixes of $(A_{\varphi(n)})_{n\in\N}$. By the pumping lemma, there is a string $A_{\varphi(0)}A_{\varphi(1)}\ldots A_{\varphi(N)}$ and two integers $k,k'$ such that \\
$A_{\varphi(0)}A_{\varphi(1)}\ldots A_{\varphi(k-1)}(A_{\varphi(k)}\ldots A_{\varphi(k')})^* A_{\varphi(k'+1)}\ldots A_{\varphi(N)}$ is in the language recognised by $\mathcal{A}$, that is, prefixes of $(A_{\varphi(n)})_{n\in\N}$. This is exactly \ref{loop}.

Let $p'=\prod_k^{k'}\p(A_{\varphi(i)})$. Then for all $n\in\N$, $p'^n$ is an upper bound of the probability of some PDT of probability $\geq p$ (take any PDT where $(A_{\varphi(k)},\ldots,A_{\varphi(k')})^n$ is in the sequence of its rules). Thus $p'^n\geq p>0\ \forall n$, which is impossible unless $p'=1$. This contradicts the hypothesis that no looping chain of rules in the grammar has probability 1.
\end{proof}

\begin{cor}
	If there is no looping chain of rules of probability $1$ in the grammar, the parser is complete.
\end{cor}

\begin{proof}
	If there is no looping chain of rules of probability $1$, then Theorem \ref{finiteness of PDTs} holds.
	
	Let $A_0,\ldots,A_n$ be the sequence of rules giving the parse of the parsed string. We'll show that for all $k\in\seg{0,n}$, the parse will have after finitely many steps $A_0\ldots A_k$ as its top-ranked hypothesis.
	\begin{proof}
		Let $p_k=\prod_{i=0}^k\p(A_i)$. Let's prove the result by recursion on $k$:
		\begin{itemize}
			\item for $k=0$, $A_0$ is an axiom so is in the priority queue from the beginning of the parse. By theorem \ref{finiteness of PDTs}, there are finitely many PDTs of probability greater than $p_0$, say $N$. Since the parser will have each of them at most once as its top-ranked hypothesis, $A_0$ will be the top-ranked hypothesis before step $N+1$.
			\item suppose that after $M$ steps, $A_0\ldots A_k$ is the top-ranked hypothesis of the parser. Then at the $(M+1)^{th}$ step, the parser will expand $A_0\ldots A_k$, and put (among others) $A_0\ldots A_k A_{k+1}$ in the priority queue. Since, by theorem \ref{finiteness of PDTs}, there are finitely many PDTs of probability greater than $p_{k+1}$, say $N$, and the parser will have each of them at most once as its top-ranked hypothesis, $A_0\ldots A_k A_{k+1}$ will be the top-ranked hypothesis before step $M+1+N+1$.
			
			This completes the recursion.
		\end{itemize}
	\end{proof}
	The result follows from the case $k=n$.
\end{proof}

\section{Conditioning the rules with the CTW algorithm}

In this section we present a way to improve the performances of the parser, using the Context Tree Weighting (CTW) algorithm, whose description and properties may be found in \cite{wil95} and an implementation in \cite{Frank}. The algorithm is originally intended to be used in data compression, but its construction of context trees allows us to use it in our parser.

The CTW algorithm uses a double mixture of context trees, and its force resides in the fact that conditional probabilities may be computed recursively, allowing for a great decrease in compuation time. Its idea is to mix together the Krichevski-Trofimov estimators for all possible context trees of depth less than a certain M, allowing for a near-optimal coding (and as such, estimate of conditional probabilities, in the sense of the Kullback-Leibler divergence).

\subsection{Quick overview of the CTW}

The setting is the following : consider a stationary ergodic source of unknown law $\mathbb{P}$. We want to estimate $\mathbb{P}$ by $\hat{\mathbb{P}}$, which will be a mixture of context tree laws.

\subsubsection{Sources with a context tree}

\begin{defi}
 A complete prefix dictionnary $\mathcal{D}$ on $\mathcal{X}$ (of cardinal $K$) is a cut of $\mathcal{X}^*$ (seen as a tree), that is, a finite subpart of $\mathcal{X}^*$ such that for all $x_{-\infty:-1}$, there is a unique $m$ such that $x_{-m:-1}\in\mathcal{D}$. Let's call $f$ its context function, defined by $f(x_{-\infty:-1})=x_{-m:-1}$. Its depth, noted $l(\mathcal{D})$, is the maximum length of its elements (or, more simply, the depth of the cut).
\end{defi}

Suppose that we have reasons to think that $\mathbb{P}$ has a context tree $\mathcal{D}$ (or at least, can be succesfully approximated by such a law), that is, $\mathbb{P}$ is stationary and for all $x_{-\infty:n}$, $\mathbb{P}(X_n=x_n|X_{-\infty:n-1}=x_{-\infty:n-1})=\mathbb{P}(X_n=x_n|f(X_{-\infty:n-1})=f(x_{-\infty:n-1}))$.

We then want to estimate the conditional probabilities for all contexts in $\mathcal{D}$. Suppose that, for a context $s$, $\theta^s$ is the law on $\mathcal{X}$, conditionaly on the context $s$. Then, for a source with context $\mathcal{D}$, of parameter $(\theta^s)_{s\in\mathcal{D}}$,

\begin{align*}
 \mathbb{P}_{\mathcal{D},\theta}(X_{1:n}=x_{1:n}|X_{-\infty:0}=x_{-\infty:0}) &= \prod_{i=1}^n\mathbb{P}_{\mathcal{D},\theta}(X_{1:n}=x_{1:n}|f(X_{-\infty:0})=f(x_{-\infty:0})) \\
									      &= \prod_{s\in\mathcal{D}}\mathbb{P}_{\theta^s}(S_s(x_{1:n};x_{-\infty:0}))
\end{align*}

where $\mathbb{P}_{\theta^s}$ is the law of a string of i.i.d. variables of law $\theta^s$, and $S_s(x_{1:n};x_{-\infty:0})$ is the substring of symbols in $x_{1:n}$ with context $s$.

The idea is to mix all those probabilities for all $\theta^s$ : for a prior distribution $\nu_{\mathcal{D}}(d\theta)=\prod_{s\in\mathcal{D}}\nu(d\theta^s)$, where $\nu$ is a measure on the set of possible $\theta^s$, the simplex $\Theta=\{(\theta_1,\ldots,\theta_K)\in[0,1]^K|\sum\theta_i=1\}$, we get:

\begin{align*}
 KT_{\mathcal{D},\nu}(x_{1:n}|x_{-\infty:0}) 	&= \int_{\Theta^\mathcal{D}}\mathbb{P}_{\mathcal{D},\theta}(x_{1:n}|x_{-\infty:0}) \\
						&= \prod_{s\in\mathcal{D}}\int_{\Theta}\mathbb{P}_{\theta^s}(S_s(x_{1:n};x_{-\infty:0}))\nu(d\theta^s)
\end{align*}

A good choice for $\nu$ is a Dirichlet $\mathbb{D}(1/2,\ldots,1/2)$ distribution: the Dirichlet Law with parameter $\alpha=(\alpha_1,\ldots,\alpha_K)$ is the law on $\Theta$ with density 
\[f(\theta_1,\ldots,\theta_K)=\frac{\Gamma(\alpha_1+\ldots+\alpha_K)}{\Gamma(\alpha_1)\ldots\Gamma(\alpha_K)}\prod_{i=1}^K\theta_i^{\alpha_i}
\]
with respect to the Lebesgue density.

 This distribution has nice properties, for example, an oracle inequality for the code-length. This choice gives the \emph{Krichevski-Trofimov} estimator $KT_{\mathcal{D}}=KT_{\mathcal{D},\mathbb{D}(1/2,\ldots,1/2)}$ for sources with a context tree.

\begin{lemme}
Let \begin{itemize}
              \item[-] $c_s^y(x_{1:n}|x_{-\infty:0})$ denote the number of $y$ in context $s$,
	      \item[-] $C_s(x_{1:n}|x_{-\infty:0})=\sum_yc_s^y(x_{1:n}|x_{-\infty:0})$, the number of occurence of context $s$.
             \end{itemize}Then:
 \[
  KT_{\mathcal{D}}(x_{1:n}|x_{-\infty:0})=\prod_{s\in\mathcal{D}}\frac{\Gamma(K/2)}{\Gamma(1/2)^K}\frac{\prod_{y\in\mathcal{X}}\Gamma(c_s^y(x_{1:n}|x_{-\infty:0})+1/2)}{\Gamma(C_s(x_{1:n}|x_{-\infty:0})+1/2)}
 \]
\end{lemme}

These quantities may be recursively computed by the following lemma (for a binary alphabet):

\begin{lemme}
 Let $P_e(a,b)$ denote the K-T estimator (giving the probability of having a $0$) for a particular context $s$, $a$ (resp $b$) representing the number of 0 (resp 1) seen in context $s$. Then $P_e(0,0)=1$, and for $a\geq0, b\geq0$,
\[
 P_e(a+1,b)=\frac{a+1/2}{a+b+1}P_e(a,b)\qquad and \qquad P_e(a,b+1)=\frac{b+1/2}{a+b+1}P_e(a,b).
\]
\end{lemme}

The proof will not be presented here, but is quite easy, and may be found in \cite{wil95}. The case of an alphabet of size $K$ is identical, but much longer to write...

\subsubsection{The Context Tree Weighting Method}

When the context tree of our source is unknown, one solution is to mix over all possible context trees (of a certain maximum depth). The Context Tree Weighting methods consists in this idea: if $\pi$ is a probability on context trees, we take

\[CTW(x_{1:n})=\sum_\mathcal{D}\pi(\mathcal{D})KT_\mathcal{D}(x_{1:n})
\]

Typically, we will take for $\pi$ a branching law, in which all nodes of depth less than a certain $M$ has probability $\alpha\leq1/K$ of having $K$ daughters.

An important result is that this method is \emph{universal}:

\begin{theo}
 If $(X_n)_{n\in\mathbb{N}}$ is ergodic, stationary of law $\mathbb{P}$, then, $\mathbb{P}$-a.s.,
\[
 \lim_{n\rightarrow\infty}-\frac{1}{n}\log CTW_\alpha(X_{1:n})=H(\mathbb{P})
\]
where $H(\mathbb{P})$ is the entropy of $\mathbb{P}$.
\end{theo}

The great advantage of this method is that it may be recursively computed.

We will now present quickly how this can be done.

Let us denote, for each context $s$ and $x\in\mathcal{X}$, $x_s$ as the number of $x$ seen in context $s$, and $P_e((x_s)_{x\in\mathcal{X}},y)$ the corresponding K-T estimator (that is, the probability of having $y\in\mathcal{X}$ in context $s$).

\begin{defi}
 To each node $s$ of the context tree $\mathcal{T}$ of depth $M$, we assign a weighted probability $\mathbb{P}_w^s$ which is defined as
\[
 \mathbb{P}_w^s(y)=\left\{\begin{array}{ll}
		  \frac{(K-1)P_e((x_s),y)+\prod_{\varphi\in\mathcal{X}}\mathbb{P}_w^{\varphi s}(y)}{K}&\text{for } 0\leq l(s)<M\\
		  P_e((x_s),y)&\text{otherwise}
                \end{array}\right.
\]
\end{defi}

This construction has the expected property, that is, $\mathbb{P}_w^s(.)=CTW(.|s)$, for the $\alpha=1/K$ mixture.

\subsection{Application for the parser}

This algorithm may be applied to our parser, to condition the rewriting rules. We can, with this method, condition on whatever we want, provided it appears as a linear string of symbols. An obvious (and perhaps a bit naïve) choice would be to condition on the string of rules already used in the parse. Another would be to condition on the string of rules \emph{descending directly to the expanded node from the root}. For example, let's see, with our cats and mouses grammar \ref{catsandmouses} the sentence `Which mouse did the cat eat' \ref{dermouse}, when the parser will try to expand the node giving `the cat' (at point `Which mouse did $\cdot$ the cat eat'): it will be for example in state (index of the rules indicated in (.))

\Tree
[.{(2)}
  [.{$\circ$\\(11)}
    [.{$\bullet$\\(12)}
      [.{Lex\\(13)}
	[.{$did :: =\text{v}+\text{wh}\ \text{c}$}
	]
      ]
      [.{$\bullet$\\(14)}
	[.{\frame{$[=\text{n}\cdot\text{d}]$}}
	]
	[.{$\bullet$\\(16)}
	  [.{$[\cdot=\text{d}=\text{d}\ \text{v}]$}
	  ]
	  [.{$\bullet$\\(18)}
	    [.{Lex\\(19)}
	      [.{$which :: =\text{n}\ \text{d}-\text{wh}$}
	      ]
	    ]
	    [.{Lex\\(9)}
	      [.{$mouse :: \text{n}$}
	      ]
	    ]
	  ]
	]
      ]   
    ]
  ]
]

and will have already constructed the following string of rules: 2-11-12-14-16-18-19-9-13. It will condition on the string of rules 2-11-12-14, meaning that:
\begin{itemize}
 \item[-] it's the subject of a VP with moving object (14),
 \item[-] it's a past sentence (12),
 \item[-] it's an interrogative sentence (11 and 2)
\end{itemize}

Indeed, only the first seems relevant, but using all string of rules will condition first by the fact that it is a past sentence, that a mouse is involved, etc... and will have to go up 6 rules to know that it is a subject that is currently expanded... which seems the important information. This conditioning allows to condition roughly on the successive heads c-commanding the expanded node (plus movement information, which is difficult to get rid of). Of course we could want a different method for lexicalisation rules, where thematic and semantic information would be better... but such a method is difficult to implement, having to insure that conditioning still gives a proper distribution.

Three points seem important to precise:
\begin{itemize}
 \item[-] First, although the CTW algorithm works on any finite alphabet, it is much more efficient on binary ones. Ron Begleiter and Ran El-Yaniv discussed a method in \cite{Beg06} to make the algorithm binary even in the case of an non-binary alphabet, by putting chains of CTWs (i.e. sorting the rules in a binary tree, and having a CTW algorithm for each branchement).
 \item[-] Second, the distribution, as can be seen in the previous examples, is very sparse... a given context can be followed by two or three different rules, even one, while the alphabet is huge, with 23 rules in the cats and mouses grammar (which is very simple). Of course, more complete grammar will induce a lot more variability in the possible rules for a given node, but the number of rules will grow too. However, the restriction of rules expanding a given node can be implemented directly in the structure of the context tree of the CTW, provided we know in advance the grammar -which is of course the case.
 \item[-] Finally, it is possible that, although the \emph{grammar} allows for a choice of rewriting rules for a given category, there is in fact no such choice (or with vanishing probability). Then it is possible to use a slightly different base estimator instead of the Krichevski-Trofimov one: the zero-redundancy estimator $P_e^{ZR}(a,b)$ defined as:
 \[
  P_e^{ZR}(a,b)=\left\{\begin{array}{ll}
                        \frac{1}{2}P_e(a,b)		&\text{for }a>0, b>0	\\
                        \frac{1}{2}P_e(a,0)+\frac{1}{4}	&\text{for }a>0, b=0	\\
                        \frac{1}{2}P_e(0,b)+\frac{1}{4}	&\text{for }a=0, b>0	\\
                        1				&\text{for }a=b=0	\\
                       \end{array}\right.
 \]
 This estimator better recognises sources generating only 0s or only 1s.
\end{itemize}

\section*{Conclusion}

The method described here permits to see Minimalist Grammars as the more `classical' and above all simpler Linear Context-Free Rewriting Systems (which don't have movement, and generate sentences \emph{top-down}), by taking a different point of view - considering derivation trees instead of derived trees. This enabled us to easily put a probability field on these grammars, and to parse them in a \emph{top-down, incremental way}, giving a progressive parse as the words of the sentence are discovered. The probability field allowed us to implement a \emph{beam-search} in the parser, pruning the different hypothesis to select only the more likely ones. This should accelerate the parser, while making it fail in identifying `garden-path'-type sentences. The use of more refined probabilistic tools as the CTW algorithm permits to have a better estimation of the real probability field, by conditionning the expanding rules by its context - here, the nature of the c-commanding heads, as required by the current linguistic theories.

\newpage

\bibliographystyle{alpha}
\bibliography{bibliographie}

\begin{thebibliography}{WST95}

\bibitem[BEY06]{Beg06}
Ron Begleiter and Ran El-Yaniv.
\newblock Superior guarantees for sequential prediction and lossless
  compression via alphabet decomposition.
\newblock {\em J. Mach. Learn. Res.}, 7:379--411, 2006.

\bibitem[Cho95]{Cho95}
N.~Chomsky.
\newblock {\em {The minimalist program}}.
\newblock Mit Pr, 1995.

\bibitem[FP]{Frank}
Erik Franken and Marcel Peeters.
\newblock Overview of the context tree weighting version 0.1 implementation.
\newblock {\em "\url{http://www.ele.tue.nl/ctw/download/ctw-v01_manual.pdf}"}.

\bibitem[Har01]{Har01}
Hendrik Harkema.
\newblock Parsing minimalist grammars.
\newblock 2001.

\bibitem[Roa01]{Roa01}
B.~Roark.
\newblock {Probabilistic top-down parsing and language modeling}.
\newblock {\em Computational Linguistics}, 27(2):249--276, 2001.

\bibitem[Sta97]{Sta97}
E.~Stabler.
\newblock {Derivational minimalism}.
\newblock {\em Logical aspects of computational linguistics}, page~68, 1997.

\bibitem[WST95]{wil95}
Frans~M.J. Willems, Yuri~M. Shtarkov, and Tjalling~J. Tjalkens.
\newblock The context tree weighting method : basic properties.
\newblock {\em IEEE-IT}, 41(3), 1995.

\end{thebibliography}

\end{document}